\documentclass[11pt]{article}

\usepackage[
  citingstyle=numbers,     
  bibliostyle=unsrtnat,    
  bibfile=references,      
]{brainlab}

\usepackage{amsmath,amsfonts,bm}

\def\eqref#1{equation~\ref{#1}}

\def\1{\bm{1}}

\DeclareMathAlphabet{\mathsfit}{\encodingdefault}{\sfdefault}{m}{sl}
\SetMathAlphabet{\mathsfit}{bold}{\encodingdefault}{\sfdefault}{bx}{n}

\newcommand{\EE}{\mathbb{E}}
\newcommand{\RR}{\mathbb{R}}

\newcommand{\CC}{\mathcal{C}}

\newcommand{\PP}{\mathbb{P}}
\newcommand{\cN}{\mathcal{N}}
\newcommand{\e}{\varepsilon}

\def\<#1,#2>{\left\langle #1,#2 \right\rangle}

\usepackage{amsmath,amsfonts,amssymb,amsthm,mathtools}
\newcommand{\eqdef}{\stackrel{\text{def}}{=}}

\usepackage{hyperref}
\usepackage{url}
\usepackage{wrapfig}
\usepackage{subfig}
\usepackage{booktabs}
\usepackage{tabularx}

\newcommand{\ours}{Grad-BANT}
\newcommand{\oursjl}{\ours{} + JL}
\newcommand{\tsA}{TS-finetuned}
\newcommand{\tsB}{TS-minimum-grad}
\newcommand{\tsC}{TS-momentum}

\newcommand{\nr}{--}

\newcommand{\ok}[1]{{\boldmath\textbf{#1}}}

\setbrainmeta{
  title={On Johnson-Lindenstrauss Transforms for Byzantine-Robust Optimization},
  authors={
    Aleksandr Shestakov\textsuperscript{1, 3}, Vladimir Aletov\textsuperscript{1, 2}, Aleksandr Beznosikov\textsuperscript{1, 2}
  },
  affiliations={
    \textsuperscript{1}Basic Research of Artificial Intelligence Laboratory \\
    \textsuperscript{2}Federated Learning Problems Laboratory\\
    \textsuperscript{3}MBZUAI 
  },
  abstract={
   The growing scale of data and models in machine learning calls for efficient optimization methods. Distributed optimization addresses this challenge, with compression mechanisms helping to reduce communication costs. A classical tool for dimensionality reduction is the Johnson-Lindenstrauss (JL) transform, which is simple to implement and offers strong geometric guarantees. In particular, JL transforms approximately preserve inner products and pairwise distances between vectors with high probability. In this paper, we exploit this property in Byzantine-robust optimization, where worker messages are weighted according to their alignment with a trusted reference gradient. We further support our theoretical findings with experimental results.
  },
}

\begin{document}
\begin{mainpart}

\section{Introduction}
The growing scale of data and computational demands of machine learning \citep{kairouz2021advances} have motivated distributed optimization methods \citep{tang2020communication} for the global objective
\begin{equation}\label{eq:distr_problem}
\min\limits_{x\in\RR^n} \left\{f(x) = \frac{1}{G}\sum\limits_{i\in\mathcal G} f_i(x) \right\},
\end{equation}
where $f_i$ denotes the local loss at worker $i$, $\mathcal G\subseteq\{1,\ldots,M\}$ is the fixed set of honest workers, and $G=|\mathcal G|\geq1$. The objective is the average loss over the honest workers, whose identities need not be known to the server; when all workers are honest, it reduces to the usual average over all $M$ workers. Distributed methods make large-scale training tractable when a single machine is insufficient, but introduce communication bottlenecks, synchronization overhead, and bandwidth limitations that can substantially hinder performance \citep{li2019convergence}.

To mitigate the communication cost, a broad range of compression techniques has been proposed \citep{mcmahan2017communication, alistarh2017qsgd}, such as quantization \citep{yu2019double}, which lowers numerical precision, and sparsification \citep{wangni2018gradient}, which communicates only the most informative coordinates. Workers transmit either compressed gradients \citep{alistarh2017qsgd} or auxiliary sequences that compensate for compression errors across iterations \citep{karimireddy2019error, richtarik2021ef21, mishchenko2024distributed}. In both cases, the gradient estimator is obtained by averaging compressed vectors.

However, some distributed setups rely not on the mean itself, but on geometric relations between gradients. A prominent example is Byzantine-robust optimization \citep{karimireddy2021learning}, in which Byzantine workers attempt to derail training by sending arbitrary adversarial messages. One remedy is to assign reliability weights to workers \citep{cao2020fltrust}, computed from inner products of local gradients with a trusted reference gradient \citep{yan2023recess, molodtsov2025trial}, so that workers not aligned with the reference direction are suppressed:
\begin{equation}\label{eq:byz_iter}
    \omega_i^t = \frac{\max\left\{\left\langle \nabla f_i(x^t), \nabla f_0(x^t)\right\rangle, 0\right\}}{\sum\limits_{j=1}^M \max\left\{\left\langle \nabla f_j(x^t), \nabla f_0(x^t)\right\rangle, 0\right\}},\quad x^{t+1} = x^t - \gamma\sum\limits_{i=1}^M \omega_i^t\nabla f_i(x^t),
\end{equation}
where $f_0$ is a loss function defined on a trusted device. In this illustrative update, the iterate is left unchanged if the denominator is zero, and Byzantine workers may replace their gradients by arbitrary messages.

Such methods depend on preserving $\ell_2$-geometry rather than exact coordinate values, which motivates compression operators that approximately preserve Euclidean structure. Johnson--Lindenstrauss (JL) transforms \citep{johnson1984extensions} map high-dimensional vectors into a lower-dimensional space while approximately preserving pairwise distances and inner products. We focus on stochastic JL transforms, as they are simple to implement in practice. They have been successfully applied in nearest-neighbor search \citep{kushilevitz1998efficient, andoni2015practical}, random partitioning \citep{kleinberg1997two}, linear regression \citep{heinze2016dual, thanei2017random}, clustering \citep{tang2005comparing, boutsidis2010random, dasgupta1999learning, urruty2007clustering, heckel2017dimensionality, guo2020randomized}, differential privacy \citep{upadhyay2014randomness, sheffet2017differentially}, and numerical linear algebra \citep{woodruff2014sketching, musco2020projection, sarlos2006improved}.

In distributed optimization and federated learning, however, JL transforms have mostly been treated as unbiased sketching operators \citep{song2021iterative, shrivastava2024sketching, dai2019hyper}, leading to guarantees in expectation, or derandomized \citep{acharya2019distributed}. These works do not capitalize on the preservation of distances and inner products, and thus do not clearly distinguish JL transforms from standard sparsification or sketching schemes.

\textbf{Our contributions.} In this paper, we
\begin{itemize}
    \item Propose a JL-based communication mechanism for Byzantine-robust distributed optimization methods that rely on gradient alignment with a trusted reference gradient;
    \item Use high-probability control of compressed honest gradients to derive convergence guarantees without convexity, under average data similarity and an explicit accuracy condition on the trusted reference gradient;
    \item Validate our theoretical findings experimentally and demonstrate the practical effectiveness of the proposed approach.
\end{itemize}

\section{Proposed Algorithm}
\subsection{Notation}

We equip $\RR^n$ with the standard Euclidean norm, $\|x\| \eqdef \<x,x>^{1/2}, x \in \RR^n$. Throughout the paper, $f : \RR^n \rightarrow \RR$ denotes a differentiable objective function.

\begin{assumption}[Smoothness]\label{def:smooth}
    Differentiable function $f_0$ has $L$-Lipschitz gradient for some $L>0$, i.e.
    \begin{equation}
    \left\|\nabla f_0(x) - \nabla f_0(y)\right\| \leq L\|x-y\|~~~\forall x,y \in \RR^n.
    \end{equation}
    Moreover, $f_0^{\inf} \eqdef \inf_{x\in\RR^n}f_0(x)>-\infty$.
\end{assumption}

\subsection{JL Transforms}

\begin{definition}[JL Transform \citep{johnson1984extensions}]\label{def:jl}
    Given $\e,\delta \in (0,1)$, a stochastic mapping $h : \RR^n \rightarrow \RR^k$ is called a JL transform if, for all $u,v\in \RR^n$,
    \begin{equation}
        \PP\left[(1-\e)\|u-v\|^2 \leq \|h(u)-h(v)\|^2 \leq (1+\e)\|u-v\|^2\right]\geq 1-\delta.
    \end{equation}
\end{definition}

In words, a JL transform preserves all pairwise distances up to a factor $1\pm\e$ with probability at least $1-\delta$, while reducing the dimension from $n$ to $k$. If $h$ is linear, then $h(0)=0$, and taking $v=0$ shows that norms are preserved as well: $(1-\e)\|u\|^2 \leq \|h(u)\|^2 \leq (1+\e)\|u\|^2$. Since $\<u,v> = \frac{1}{4}\left(\|u+v\|^2-\|u-v\|^2\right)$, preserving the norms of $u+v$ and $u-v$ also preserves inner products.

\begin{lemma}
    If $h$ is a linear JL Transform, then
    $\PP\left[~|\<h(u),h(v)> - \<u,v>|\leq \e\|u\|~\|v\|~\right] \geq 1-2\delta$ for all $u,v\in \RR^n$.
\end{lemma}

We consider two standard linear constructions.

\begin{definition}[Gaussian JL matrix]\label{def:gjl}
$S \in \RR^{k \times n}$ with i.i.d.\ entries $[S]_{ij} \sim \cN(0, \frac{1}{k})$.
\end{definition}
\begin{definition}[Rademacher JL matrix]\label{def:rjl}
$S \in \RR^{k \times n}$ with i.i.d.\ entries $[S]_{ij} = \pm\frac{1}{\sqrt{k}}$, each with probability $\frac{1}{2}$.
\end{definition}

Both are JL transforms with distortion $\e\in(0,1)$ and failure probability $\delta\in(0,1)$ whenever $k \geq C\e^{-2}\log (2/\delta)$, where $C>0$ is a universal constant \citep{dasgupta2003elementary, achlioptas2003database}. To return to $\RR^n$ for the optimization update, we use the lifted operator $\textup{JL}_S(x) \eqdef S^\top Sx \in \RR^n$.

\begin{lemma}[Unbiasedness]\label{lem:unbias}
    For $S$ as in Definitions \ref{def:gjl} and \ref{def:rjl}, $\EE\, \textup{JL}_S(x) = x$ for all $x \in \RR^n$.
\end{lemma}
\begin{lemma}[Second moment]\label{lem:variance}
    For all $x \in \RR^n$, $\EE \|\textup{JL}_{S}(x)\|^2 = \frac{n+k+1}{k}\|x\|^2$ for Gaussian and $\EE \|\textup{JL}_{S}(x)\|^2 = \frac{n+k-1}{k}\|x\|^2$ for Rademacher $S$.
\end{lemma}

Thus $\textup{JL}_S$ resembles unbiased sparsification such as \texttt{Randk}, $\CC(x) = \frac{n}{k}\sum_{i\in I}x_ie_i$ with $I \subset [n]$ a uniformly random subset of size $k$, which is also unbiased with variance of order $\frac{n}{k}$ \citep{beznosikov2023biased}. The key distinction is that the transmitted message of \texttt{Randk}, $h_\CC(x) \eqdef \sqrt{n/k}\,(x_i)_{i\in I} \in \RR^k$ (so that $\CC = h_\CC^\top h_\CC$, just as $\textup{JL}_S = S^\top S$), does not preserve norms with high probability. By linearity, the JL property for $h_\CC$ reduces to $\PP\big[\big|\|h_\CC(x)\|^2/\|x\|^2-1\big| \leq \e\big] \geq 1-\delta$ for all $x\in\RR^n\setminus\{0\}$. For $x = e_1$ the ratio is $0$ with probability $1-\frac{k}{n}$, forcing $\delta \geq 1-\frac{k}{n}$.

\begin{lemma}[Operator-norm bound \citep{Vershynin_2026}]\label{lem:op_norm_bound}
    For $S$ as in Definitions~\ref{def:gjl} or~\ref{def:rjl}, there exists a universal constant $c>0$, independent of $n$ and $k$, such that, with probability at least $1-\delta$,
    \begin{equation}
        \|S\| \leq 1 + c\sqrt{\frac{n}{k}} + c\sqrt{\frac{\log (2/\delta)}{k}}.
    \end{equation}
    Consequently, on the same event, $\|\textup{JL}_S(x)\|^2 \leq \left(1 + c\sqrt{\frac{n}{k}} + c\sqrt{\frac{\log (2/\delta)}{k}}\right)^4\|x\|^2$ for every $x\in\RR^n$.
\end{lemma}

The operator norm controls the lifted update, while the JL property controls the geometry of fixed vectors in the compressed space; the high-probability analysis uses both.
 

\subsection{\texttt{Grad-BANT}}

Many Byzantine-resilient methods exploit gradient similarity to distinguish honest workers from adversarial ones. We introduce \texttt{Grad-BANT} (Algorithm~\ref{alg:grad_bant}), inspired by \citep{yan2023recess}. As in \eqref{eq:byz_iter}, it relies on a trial function $f_0$ stored on the trusted server, whose gradient $\nabla f_0(x)$ is used to score the received messages; its accuracy with respect to the honest objective is specified in Assumption~\ref{as:reference_accuracy}.

 \begin{algorithm}{\texttt{Grad-BANT}}
\label{alg:grad_bant}
\begin{algorithmic}[1]
    \State \textbf{Parameters:} $x^0 \in \RR^n$, $\omega_i^{-1} = 1/M$, stepsize $\gamma>0$, momentum $\beta \in (0,1]$, smoothness constant $L>0$
    \For{$t=0,1,\dots,T-1$}
        \State Generate a fresh $S_t$; server sends $x^t$ and the sketch seed to workers
        \For{all $i=1,\dots,M$ in parallel}
            \State Honest worker: $g_i^t = \nabla f_i(x^t)$; send $y_i^t=S_t g_i^t$ to server \label{gb_stoch_upd}
        \EndFor
        \State $u_i^t=S_t^\top y_i^t$; $\theta_i^t = \<S_t\nabla f_0(x^t), y_i^t> - \tfrac{L\gamma}{2}\|u_i^t\|^2$
        \If{all $[\theta_i^t]_0=0$}
            \State $\omega_i^t = (1-\beta)\omega_i^{t-1} + \tfrac{\beta}{M}$ \label{gd_elseif}
        \Else
            \State $\omega_i^t = (1-\beta)\omega_i^{t-1} + \beta
                   \dfrac{[\theta_i^t]_0}{\sum_{j=1}^M [\theta_j^t]_0}$ \label{gd_comparison}
        \EndIf
        \State $x^{t+1} = x^t - \gamma S_t^\top
               \sum_{i=1}^M \mathbb{I}_{[\theta_i^t > 0]} \omega_i^t y_i^t$ \label{gd_step}
    \EndFor
\end{algorithmic}
\end{algorithm}

At each iteration, an honest worker sends $y_i^t=S_tg_i^t$ with $g_i^t=\nabla f_i(x^t)$, while a Byzantine worker may send any finite, measurable $y_i^t\in\RR^k$, possibly chosen after observing the current sketch and the honest messages; $u_i^t=S_t^\top y_i^t$ denotes the lifted message. Each sketch is sampled independently of all data and previous randomness after $x^t$ is determined, so the honest gradients and the reference gradient are fixed conditional on the history before sampling $S_t$.

The trust score $\theta_i^t$ measures how well the lifted message aligns with the trusted reference gradient $\nabla f_0(x^t)$, while accounting for its norm: if $\<\nabla f_0(x^t), u_i^t> > \frac{L\gamma}{2}\|u_i^t\|^2$, the direction is sufficiently aligned with the trusted descent direction and receives a positive weight; otherwise, it is suppressed by the indicator in the update. The scores are clipped by $[x]_0 \eqdef \max\{x,0\}$ and normalized. If all scores are non-positive, we blend the previous weights with uniform weights (line~\ref{gd_elseif}) and leave the iterate unchanged. The momentum parameter $\beta$ smooths the evolution of the weights over time.

For honest workers, the server thus uses the compressed alignment terms $\<S_t\nabla f_0(x^t), S_tg_i^t>$, which the JL property keeps close to the exact ones with high probability, while the norm penalty is evaluated on the actual lifted message. The aggregated update is lifted back to $\RR^n$ via $S_t^\top$ (line~\ref{gd_step}). Our analysis uses the aggregate honest alignment and the sum of honest squared norms; it does not require the compressed and uncompressed scores to select the same workers. We require the following similarity condition.

\begin{assumption}[Data similarity]\label{as:data_sim}
    The honest workers are said to satisfy $(\delta_1,\delta_2)$-data similarity if there exist $\delta_1,\delta_2\geq0$ such that
    \begin{equation}
        \frac{1}{G}\sum_{i\in\mathcal G}\left\|\nabla f_i(x) - \nabla f(x)\right\|^2 \leq \delta_1  + \delta_2 \left\|\nabla f(x)\right\|^2,
        \qquad \forall x \in \RR^n.
    \end{equation}
\end{assumption}

This condition bounds heterogeneity on average over the honest workers, which is weaker than imposing it on every honest worker. Since $f$ is their average, it also gives
\begin{equation}\label{eq:honest_second_moment}
    \frac{1}{G}\sum_{i\in\mathcal G}\|\nabla f_i(x)\|^2
    =\|\nabla f(x)\|^2+\frac{1}{G}\sum_{i\in\mathcal G}\|\nabla f_i(x)-\nabla f(x)\|^2
    \leq (1+\delta_2)\|\nabla f(x)\|^2+\delta_1.
\end{equation}

The next assumption quantifies how accurately the trusted reference represents the objective; $\eta_0$ and $\eta_1$ capture the mismatch between the trial and the true data distributions.

\begin{assumption}[Reference accuracy]\label{as:reference_accuracy}
    There exist $\eta_0\geq0$ and $\eta_1\in[0,1)$ such that
    \begin{equation}
        \|\nabla f_0(x)-\nabla f(x)\|^2
        \leq \eta_0+\eta_1\|\nabla f(x)\|^2,
        \qquad \forall x\in\RR^n.
    \end{equation}
\end{assumption}
\begin{theorem}\label{th:byz}
Suppose that Assumptions~\ref{def:smooth}, \ref{as:data_sim}, and~\ref{as:reference_accuracy} hold. Let $T\geq1$ be an integer, $\beta\in(0,1]$, and $\e,\delta\in(0,1)$ satisfy $d\eqdef1-\eta_1-\e^2>0$. Use fresh Gaussian or Rademacher JL matrices in Algorithm~\ref{alg:grad_bant}, and set
\begin{equation}
    \rho=\left(1+c\sqrt{\frac{n}{k}}+c\sqrt{\frac{\log(8T/\delta)}{k}}\right)^2,
    \qquad \Delta_0=f_0(x^0)-f_0^{\inf},
\end{equation}
where $c$ is the constant in Lemma~\ref{lem:op_norm_bound}. Then, for any stepsize
\begin{equation}
    0<\gamma\leq\frac{d}{2L\rho(1+\e)(1+\delta_2)},
\end{equation}
after $T$ iterations we have
\begin{equation}\label{eq:main_convergence}
    \frac{1}{T}\sum_{t=0}^{T-1}\|\nabla f(x^t)\|^2
    \leq \frac{4M\Delta_0}{\gamma\beta GdT}
    +\frac{2\eta_0}{d}
    +\frac{2L\gamma\rho(1+\e)\delta_1}{d},
\end{equation}
with probability at least $1-\delta$, provided that $k\geq C\e^{-2}\log(8T/\delta)$ for a sufficiently large universal constant $C>0$.
\end{theorem}

The proof is given in Appendix~\ref{appendix:byz}. Neither convexity nor a small upper bound on $\delta_2$ is required \cite{molodtsov2025trial}. The reference error contributes an explicit residual term, heterogeneity affects both the admissible stepsize and the last term in~\eqref{eq:main_convergence}, and the factor $M/(\beta G)$ reflects the fraction of honest workers and the weight update. With $S_t=I_n$ ($\e=0$, $\rho=1$), the same analysis gives the uncompressed guarantee.

\section{Efficient Implementation}

Beyond communication, scalability in large-scale distributed learning is limited by the cost of storing and applying the compression operator: a dense Gaussian or Rademacher matrix $S \in \RR^{k \times n}$ requires $O(nk)$ memory and $O(nk)$ operations per application, which is prohibitive for modern models. We therefore discuss structured and sparse JL constructions, whose geometric guarantees depend on the construction and its parameters, and their use within \texttt{Grad-BANT}.

\subsection{Efficient JL Transforms}

Dense JL matrices can be replaced by the SRHT $S_{\mathrm{SRHT}} = \sqrt{n/k}\,R_THD$, with a random diagonal sign matrix $D$, the normalized Walsh--Hadamard matrix $H$, and a sampler $R_T$ of $k$ coordinates, or by the related FJLT \citep{ailon2006approximate,tropp2011improved,ailon2013almost}, both applied in $O(n\log n)$ time, or by sparse sketches such as CountSketch \citep{charikar2002finding,clarkson2013low,kane2014sparser}, applied in $O(n)$ time. All of them are generated from a random seed, so the server broadcasts only the seed of $S_t$. The appendix verifies unbiasedness and second-moment bounds for these sketches, which suffices for the expectation analysis (Theorem~\ref{th:byz_expectation}); the high-probability guarantee of Theorem~\ref{th:byz} additionally requires the inner-product, Frobenius-norm, and operator-norm events of Appendix~\ref{appendix:byz}.

\subsection{Per-layer Sketching}

Applying a single JL transform to the entire parameter vector may be impractical for large models, and even infeasible for large language models. A natural remedy is per-layer sketching, where each layer $x_w\in\RR^{n_w}$ of $x=(x_1,\ldots,x_W)$ is compressed by its own JL transform $S_w\in\RR^{k_w\times n_w}$, i.e., $S=\operatorname{blockdiag}(S_1,\ldots,S_W)$. If each layerwise transform preserves norms, so does the full block-diagonal transform.

\begin{theorem}[Per-layer norm preservation]\label{th:per_layer_jl}
    Let $x = (x_1,\ldots,x_W) \in \RR^n \setminus \{0\}$, where
    $x_w \in \RR^{n_w}$ and $\sum_{w=1}^W n_w = n$. Let
    $S = \operatorname{blockdiag}(S_1,\ldots,S_W)$, where
    $S_w \in \RR^{k_w \times n_w}$ are random matrices defined on a common
    probability space. Suppose that, for each $w$, there exist parameters
    $\e_w \geq 0$ and $\delta_w \in [0,1]$ such that, for every fixed
    $z \in \RR^{n_w}$, $
        \PP\left(
        \Big|\|S_wz\|^2 - \|z\|^2 \Big|
        >
        \e_w\|z\|^2
        \right)
        \leq \delta_w.$
    Then, for every fixed $x \in \RR^n \setminus \{0\}$,
    \begin{equation*}
        \PP\left(
        \Big|\|Sx\|^2 - \|x\|^2\Big|
        >
        \left(\sum_{w=1}^W \alpha_w(x)\e_w\right)\|x\|^2
        \right)
        \leq
        \sum_{w=1}^W \delta_w,
    \end{equation*}
    where $
        \alpha_w(x) \eqdef {\|x_w\|^2}/{\|x\|^2}.
    $
\end{theorem}

If the layerwise distribution of gradient energy is known, the sketch dimensions can be allocated non-uniformly across layers.

\begin{corollary}[Energy-aware allocation]\label{cor:energy_aware}
    Let $\e,\delta\in(0,1)$ and use Gaussian or Rademacher sketches on the active layers, with dimensions $
        k_w =
        \left\lceil
        C\frac{\log(2W/\delta)}{\e_w^2}
        \right\rceil,
    $
    where $C>0$ is a sufficiently large universal constant.
    Fix $x \in \RR^n\setminus\{0\}$ and let
    $\alpha_w(x)=\|x_w\|^2/\|x\|^2$. For all active layers with
    $\alpha_w(x)>0$, set
    \begin{equation}\label{eq:en_aware}
        \e_w
        =
        \frac{\e\,\alpha_w(x)^{-1/3}}
        {\sum_{j:\alpha_j(x)>0}\alpha_j(x)^{2/3}},
        \qquad
        \delta_w = \frac{\delta}{W}.
    \end{equation}
    Assume that $\e_w<1$ and $k_w\leq n_w$ for every active layer. Inactive
    layers are omitted, with $k_w=0$. Then, for this fixed vector $x$, the
    block-diagonal sketch satisfies an $\e$-distortion guarantee with
    probability at least $1-\delta$, and the total number of transmitted
    coordinates scales as
    \begin{equation*}
        \sum_{w=1}^W k_w
        =
        \Theta\left(
        \frac{\log(2W/\delta)}{\e^2}
        \left(\sum_{w=1}^W \alpha_w(x)^{2/3}\right)^3
        \right).
    \end{equation*}
\end{corollary}

In this feasible regime, energy-aware allocation can reduce the total communication cost for the fixed vector compared to a uniform allocation.

\subsection{Matrix and Tensor Extensions}

Neural network parameters are naturally matrices or higher-order tensors, which suggests sketches that respect this structure. For $X \in \RR^{n_1 \times n_2}$, instead of applying a dense sketch to $\operatorname{vec}(X)$, one can use two smaller JL matrices $S_1 \in \RR^{k_1 \times n_1}$ and $S_2 \in \RR^{k_2 \times n_2}$ and set
\[
    \mathcal{S}(X) = S_1XS_2^\top \in \RR^{k_1 \times k_2},
    \qquad
    \operatorname{vec}(\mathcal{S}(X)) = (S_2 \otimes S_1)\operatorname{vec}(X).
\]
A dense sketch of $\operatorname{vec}(X)\in\RR^{n_1n_2}$ to dimension $k$ requires $O(n_1n_2k)$ memory, whereas the Kronecker-structured sketch stores only $S_1$ and $S_2$, i.e., $O(n_1k_1+n_2k_2)$ memory, for a compressed dimension $k_1k_2$. This dimension may not be optimal, but the savings in memory and implementation cost are essential for large-scale models. The same idea extends to higher-order tensors.

\begin{theorem}[Kronecker-structured JL transform]\label{th:kron_jl}
    Let $S = S_d \otimes S_{d-1}\otimes \ldots \otimes S_1$, where $S_j \in \RR^{k_j \times n_j}$ are independent Gaussian or
    Rademacher JL matrices, and suppose that $\min_{j\in[d]} k_j = \Omega\left(\frac{d^2\log(2d/\delta)}{\e^2}\right)$.
    Then, for every fixed tensor $X \in \RR^{n_1 \times \ldots \times n_d}$, with $\mathcal{S}(X)$ defined by $\operatorname{vec}(\mathcal{S}(X)) = S\operatorname{vec}(X)$,
    \begin{equation*}
        \PP \left(
        \Big|\|\mathcal{S}(X)\|_F^2 - \|X\|_F^2 \Big|
        >
        \e \|X\|_F^2
        \right)
        \leq \delta.
    \end{equation*}
\end{theorem}

As for per-layer sketches, their use in the high-probability convergence proof requires the theoretical justifying specified in Appendix~\ref{appendix:byz}. Overall, dense Gaussian and Rademacher projections are useful for theory and small-scale experiments, whereas structured, sparse, per-layer, and Kronecker-based sketches reduce the cost of applying and storing the compression operator.

\section{Experiments}

\subsection{Comparison of Compressors}
\label{sec:compressor-comparison}

In this section we compare several compressors along two axes:
the time and memory required to compress and decompress the transmitted parameters and
the extent to which they satisfy the JL-properties required by our Algorithm~\ref{alg:grad_bant}.

Table~\ref{tab:cost} reports the cost of applying each compressor to the parameters of a ResNet18 model.
For \textsc{RandK} we treat all parameters as a single flattened vector, whereas for the JL-based compressors and \textsc{PowerSGD} we compress each layer separately.
For \textsc{SRFT} and \textsc{SRHT} we consider both scenarios, since these methods can operate in either one without incurring a significant memory or time overhead compared to the remaining JL compressors.
CountSketch is omitted from the experiments, as it is not readily supported by standard PyTorch tensor operations.
\emph{Stored} denotes the additional memory required to keep the compressor's internal state (e.g.\ the JL matrix).
All measurements are taken at a compression rate of $50\%$.
\begin{table}[h!]
\centering
\caption{Cost of applying each compressor to ResNet-18 parameters at $k/n = 0.5$: compression and decompression time, peak memory, memory for the compressor state (Stored) and per-client uplink message size. JL compressors and PowerSGD are applied layer-wise, Rand-$k$ to the flattened parameter vector; $^\dagger$~applied to the flattened parameter vector.}
\label{tab:cost}
\small
\setlength{\tabcolsep}{5pt}
\begin{tabular}{l r r r r r}
\toprule
Compressor & GPU, ms & CPU, ms & Peak, MB & Stored, MB & Uplink, MB \\
\midrule
Gaussian JL, Kronecker (ours)   & $6.5$   & \ok{$155$} & \ok{$92.5$} & $69.9$  & $21.3$ \\
Rademacher JL, Kronecker (ours) & $8.1$   & $160$      & $92.5$      & $69.9$  & $21.3$ \\
SRHT                            & $55.0$  & $718$      & $181.2$     & $117.9$ & $21.3$ \\
SRFT                            & $12.7$  & $659$      & $131.1$     & $63.9$  & $21.3$ \\
SRHT, whole vector$^\dagger$    & $30.6$  & $927$      & $644.3$     & $106.6$ & $21.3$ \\
SRFT, whole vector$^\dagger$    & $8.8$   & $693$      & $432.0$     & $64.0$  & $21.3$ \\
Rand-$k$                        & \ok{$3.8$} & $374$   & $603.4$     & $0.0$   & $21.3$ \\
PowerSGD (rank $230$)           & $174.6$ & $195$      & $94.7$      & $18.7$  & $21.3$ \\
Gaussian JL$^\dagger$, dense$^\ddagger$   & $ 12 \cdot 10^{3}$      & --         & $3.7\cdot10^{7}$ & $3.7\cdot10^{7}$ & -- \\
\bottomrule
\end{tabular}
\end{table}

As the table shows, our JL variant is the fastest among all JL-based methods, and is outperformed only by \textsc{RandK} on GPU, which is unsurprising since that method merely selects a subset of coordinates.
At the same time, in terms of peak memory and CPU time our method outperforms every compressor considered.
Applying \textsc{SRHT} and \textsc{SRFT} to the full parameter vector offers no practical advantage over layer-wise application. The effect on running time is marginal (a minor speedup on GPU and a $5$--$30\%$ slowdown on CPU), whereas peak memory consumption increases by a factor of $3.3$--$3.6$.

Table~\ref{tab:adaptive-eps} compares several quantities that measure how well the JL properties are preserved by the different compression methods.
Here $\Delta_{\mathrm{ip}}$ and $\Delta_{\theta}$ denote the quantities
\begin{equation*}
\Delta_{\mathrm{ip}} = \frac{\bigl|\langle Su, Sv\rangle - \langle u, v\rangle\bigr|}{\|u\|\,\|v\|},
\qquad
\Delta_{\theta} = \bigl|\angle(Su, Sv) - \angle(u, v)\bigr|.
\end{equation*}
$\Delta_{\mathrm{ip}}$ measures the deviation of the inner product of the compressed and uncompressed vectors, whereas the $\Delta_{\theta}$ measures the distortion of the angle between the vectors caused by compression.
\begin{table}[h!]
\centering
\small
\caption{Distortion of inner products ($\Delta_{\mathrm{ip}}$) and angles ($\Delta_\theta$) under compression, mean $\pm$ std. (a)~Random $4096\times4096$; (b)~stochastic ResNet-18 gradient vs.\ trust-set gradient. All compressors use $k/n = 0.5$; for Adaptive JL, $k/n$ is set implicitly by the target $\varepsilon$ from~\eqref{eq:en_aware}.}
\label{tab:adaptive-eps}
\resizebox{\textwidth}{!}{%
\small
\setlength{\tabcolsep}{4pt}
\begin{tabular}{l l r r r @{\hspace{1.5em}} r r r}
\toprule
& & \multicolumn{3}{c}{\textit{(a) random matrix $4096\times4096$}} & \multicolumn{3}{c}{\textit{(b) real ResNet-18 gradients}} \\
& & \multicolumn{3}{c}{\textit{20 draws of $(S,u,v)$, mean $\pm$ std}} & \multicolumn{3}{c}{\textit{100 pairs, mean $\pm$ std}} \\
\cmidrule(lr){3-5}\cmidrule(lr){6-8}
Compressor & Variant & $k/n$ & $\Delta_{\mathrm{ip}}$, $10^{-4}$ & $\Delta_{\theta}$, deg & $k/n$ & $\Delta_{\mathrm{ip}}$, $10^{-4}$ & $\Delta_{\theta}$, deg \\
\midrule
Kronecker JL (ours) & Rad.   & $0.500$ & $5.0\pm4.5$ & $0.011\pm0.0090$ & $0.500$ & $194\pm156$ & $1.1\pm1.7$ \\
             & Gauss. & $0.500$ & $8.7\pm8.3$ & $0.017\pm0.013$  & $0.500$ & $94\pm86$ & $0.8\pm0.7$ \\
\addlinespace
SRHT & layer & $0.500$ & $2.5\pm1.9$ & $0.0060\pm0.0056$ & $0.500$ & $103\pm80$  & $0.55\pm0.43$ \\
     & whole & --      & --          & --                & $0.500$ & $302\pm251$ & $1.31\pm1.23$ \\
\addlinespace
SRFT & layer & $0.500$ & $2.6\pm1.7$ & $0.0067\pm0.0064$ & $0.500$ & $120\pm125$ & $0.63\pm0.58$ \\
     & whole & --      & --          & --                & $0.500$ & $470\pm198$ & $2.1\pm1.11$ \\
\addlinespace
Rand-$k$ & & $0.500$ & $2.1\pm2.0$ & $0.0063\pm0.0064$ & $0.500$ & $2000\pm800$ & $1.4\pm1.1$ \\
\addlinespace
Adaptive JL & target $\varepsilon=0.05$ & $1.00$   & $0$          & $0$               & $1.00$  & $0$         & $0$ \\
            & target $\varepsilon=0.1$  & $0.841$  & $3.3\pm2.5$  & $0.011\pm0.012$   & $1.00$  & $0$         & $0$ \\
            & target $\varepsilon=0.2$  & $0.0610$ & $13\pm11$    & $0.025\pm0.021$   & $0.605$ & $39\pm38$   & $0.23\pm0.21$ \\
            & target $\varepsilon=0.3$  & $0.0141$ & $19\pm12$    & $0.057\pm0.050$   & $0.240$ & $96\pm77$   & $0.51\pm0.43$ \\
\bottomrule
\end{tabular}}
\end{table}

The table reports results for two setups: setup~(a) is the compression of two $4096 \times 4096$ matrices, and setup~(b) is the compression of an actual stochastic gradient of a ResNet18 model and its comparison against the gradient computed on the trust set, as used in Algorithm~\ref{alg:grad_bant}.
For all compressors except \texttt{AdaptiveJL} the compression rate is fixed at $50\%$.
For \texttt{AdaptiveJL} the compression rate is determined implicitly by the parameter $\varepsilon$. We therefore report results for several values of $\varepsilon$.
A comparison of the compressors at other compression rates is provided in Appendix~\ref{app:exps}.

As the table shows, \textsc{RandK} performs reasonably well in the first setup, which can be attributed to the random construction of the matrices and to the unbiasedness of the compressor. However, 
on real gradients the method distorts both the inner product and the angle quite severely.
In contrast, the JL-based methods perform well in both scenarios. We also note separately that \textsc{SRFT} and \textsc{SRHT} applied to the gradient as a single vector yield the worst results among all JL-based methods considered, which together with Table~\ref{tab:cost}, indicates that layer-wise compression is preferable.

\subsection{Comparison with Byzantine Baselines}
\label{sec:comp_methods}

This section presents experiments comparing Algorithm~\ref{alg:grad_bant} against a range of federated learning methods under various Byzantine attack scenarios.
All results are reported in the tables below.
Every experiment is run in the setting of training a ResNet18 model on the CIFAR-10 dataset.
All methods are trained for $T = 200$ communication rounds with the data distributed uniformly across $M = 10$ clients.
To compare methods with different compression levels, we measure communication in \textit{full gradient uploads}: a compressed message of $k$ coordinates counts as $k/n$ of a full gradient, where $n$ is the number of model parameters, and uploads are summed over all $M$ clients.
After $t$ rounds this gives
\begin{equation}
    G(t) = \sum_{s=1}^{t}\sum_{i=1}^{M} \frac{k_{s,i}}{n},
    \label{eq:full-grad-uploads}
\end{equation}
which reduces to $G(t) = tM\,k/n$ for a fixed compression ratio and to $G(t) = tM$ without compression.
Over the full $T = 200$ rounds an uncompressed method thus performs $2000$ full-gradient uploads, and a method with $k/n = 0.5$ performs $1000$. Details on the attack scenarios and additional experiments are deferred to Appendix~\ref{app:exps}.

\begin{table}[h!]
\centering
\begin{minipage}[c]{0.49\linewidth}
\centering
\caption{FedAvg and \ours{} with different compressors: (a)~no attack, (b)~label flipping with $60\%$ Byzantine clients.}
\label{tab:compressors}
\setlength{\tabcolsep}{4pt}
\resizebox{\linewidth}{!}{%
\begin{tabular}{l l c c r r r}
\toprule
Method & Compressor & $k/n$ & Byz. & G$\to 60$ & G$\to 80$ & Best acc \\
\midrule
\multicolumn{7}{l}{\textit{(a) no attack}} \\
FedAvg & Gaussian JL & $0.1$ & $0$ & $23$ & $52$ & $0.880$ \\
FedAvg & Rand-$k$    & $0.1$ & $0$ & $15$ & $39$ & $0.897$ \\
FedAvg & SRHT        & $0.1$ & $0$ & $17$ & $42$ & $0.895$ \\
\addlinespace
\ours{} & Gaussian JL & $0.1$ & $0$ & $18$ & $53$ & $0.896$ \\
\ours{} & Rand-$k$    & $0.1$ & $0$ & $15$ & $41$ & $0.891$ \\
\ours{} & SRHT        & $0.1$ & $0$ & $16$ & $42$ & $0.893$ \\
\midrule
\multicolumn{7}{l}{\textit{(b) label flipping}} \\
FedAvg & none        & $1.0$ & $60\%$ & \nr & \nr & $0.276$ \\
FedAvg & Gaussian JL & $0.1$ & $60\%$ & \nr & \nr & $0.247$ \\
FedAvg & Rand-$k$    & $0.1$ & $60\%$ & \nr & \nr & $0.315$ \\
FedAvg & SRHT        & $0.1$ & $60\%$ & \nr & \nr & $0.334$ \\
\addlinespace
\ours{} & Gaussian JL & $0.1$  & $60\%$ & $31$ & $71$ & $0.838$ \\
\ours{} & Adaptive JL & $0.24$ & $60\%$ & $25$ & $69$ & $0.833$ \\
\ours{} & SRHT        & $0.1$  & $60\%$ & $26$ & $78$ & $0.830$ \\
\ours{} & Rand-$k$    & $0.1$  & $60\%$ & \nr  & \nr  & $0.546$ \\
\bottomrule
\end{tabular}}
\end{minipage}\hfill
\begin{minipage}[c]{0.49\linewidth}
\centering
\caption{Best test accuracy of Byzantine-robust methods with compression under label flipping (LF) and random gradients (RG) with $40\%$ and $60\%$ Byzantine clients. All methods use the $k/n = 0.5$.}
\label{tab:compressed-byz}
\setlength{\tabcolsep}{5pt}
\resizebox{\linewidth}{!}{%
\begin{tabular}{l rrrr}
\toprule
Method & LF $40\%$ & LF $60\%$ & RG $40\%$ & RG $60\%$ \\
\midrule
RoSDHB             & $0.595$ & $0.331$ & $0.308$ & $0.103$ \\
Byz-EF21           & $0.430$ & $0.294$ & $0.441$ & $0.123$ \\
Byz-EF21-BC        & $0.429$ & $0.265$ & $0.431$ & $0.103$ \\
Byz-VR-MARINA~2.0  & $0.555$ & $0.159$ & $0.571$ & $0.123$ \\
Byz-DASHA-PAGE     & $0.565$ & $0.152$ & $0.455$ & $0.115$ \\
\oursjl{} (ours)   & \ok{$0.883$} & \ok{$0.862$} & \ok{$0.620$} & \ok{$0.534$} \\
\bottomrule
\end{tabular}}
\end{minipage}
\end{table}

Table~\ref{tab:compressors} reports the results of applying different compressors to \texttt{FedAvg} and \texttt{GradBant}, both in the attack-free setting and under the \textsc{Label Flipping} (LF) attack.
For each configuration, the table reports the number of full-gradient transmissions required to reach a given value of the target metric, together with the best test accuracy attained.
Columns G$\to$60 and G$\to$80 report $G(t)$ from~\eqref{eq:full-grad-uploads} at the first round $t$ at which the test accuracy reaches $60\%$ and $80\%$, respectively. For \texttt{AdaptiveJL}, $k/n$ denotes the average over rounds and clients.

In the attack-free setting, all compressors perform similarly for both algorithms.
Once the attack is enabled, every compressor in the \texttt{FedAvg} setting fails to learn and shows almost no accuracy improvement throughout training.
In contrast, for \texttt{GradBant} the properties of the compressor start to matter. All JL compressors which preserve inner products continue to work, whereas the \texttt{RandK} compressor is likewise compromised by the attack.

Table~\ref{tab:compressed-byz} compares Byzantine-robust methods with compression \cite{rammal2024communication} under two different attacks.
Unlike ours, the remaining methods do not use a trust set, and consequently cannot tolerate a Byzantine majority.
When the number of Byzantine clients approaches one half, these methods yield results that lag far behind our algorithm.

\begin{table}[h!]
\centering
\caption{Comparison with Byzantine-robust baselines under different attacks (Byzantine fraction in the header): (a)~best test accuracy; (b)~full-gradient uploads $G$ to reach $60\%$ test accuracy. Baselines are uncompressed; \oursjl{} uses $k/n = 0.5$. \nr{}: not reached.}
\label{tab:baselines}
\small
\setlength{\tabcolsep}{3.5pt}
\begin{tabular}{l rrrrrr}
\toprule
Method & none & LF $60\%$ & SF $60\%$ & RG $40\%$ & ALIE $60\%$ & IPM $70\%$ \\
\midrule
\multicolumn{7}{l}{\textit{(a) best accuracy}} \\
FedAvg (Adam)  & $0.917$ & $0.276$ & $0.331$ & $0.132$ & $0.145$ & $0.100$ \\
CC             & \ok{$0.928$} & $0.828$ & $0.574$ & $0.122$ & $0.143$ & $0.100$ \\
CC + bucketing & $0.840$ & \ok{$0.897$} & \ok{$0.891$} & $0.134$ & $0.143$ & $0.100$ \\
CC + FBM       & $0.924$ & $0.855$ & \ok{$0.892$} & $0.116$ & $0.184$ & $0.100$ \\
Safeguard      & $0.921$ & $0.104$ & $0.858$ & $0.132$ & $0.113$ & $0.137$ \\
\cmidrule(lr){1-7}
FLTrust        & $0.767$ & $0.704$ & $0.594$ & $0.100$ & $0.100$ & $0.101$ \\
RECESS         & $0.889$ & $0.681$ & $0.512$ & $0.142$ & $0.146$ & $0.106$ \\
Zeno           & $0.919$ & $0.361$ & $0.859$ & $0.105$ & $0.101$ & $0.124$ \\
\tsA{}         & $0.920$ & $0.884$ & $0.877$ & $0.889$ & \ok{$0.903$} & $0.630$ \\
\tsB{}         & $0.922$ & $0.887$ & $0.861$ & \ok{$0.901$} & $0.882$ & $0.677$ \\
\tsC{}         & $0.868$ & $0.863$ & $0.860$ & $0.857$ & $0.865$ & $0.700$ \\
\oursjl{} (ours), $50\%$ compression & $0.910$ & $0.862$ & $0.878$ & $0.620$ & $0.539$ & \ok{$0.827$} \\
\midrule
\multicolumn{7}{l}{\textit{(b) full-gradient uploads $G$ to $60\%$}} \\
FedAvg (Adam)  & $160$ & \nr & $850$ & \nr & \nr & \nr \\
CC             & \ok{$60$} & $290$ & \nr & \nr & \nr & \nr \\
CC + bucketing & $460$ & $360$ & $130$ & \nr & \nr & \nr \\
CC + FBM       & \ok{$60$} & $140$ & $130$ & \nr & \nr & \nr \\
Safeguard      & $140$ & \nr & $170$ & \nr & \nr & \nr \\
\cmidrule(lr){1-7}
FLTrust        & $840$ & $1140$ & \nr & \nr & \nr & \nr \\
RECESS         & $150$ & $610$ & \nr & \nr & \nr & \nr \\
Zeno           & $160$ & \nr & $220$ & \nr & \nr & \nr \\
\tsA{}         & $130$ & $250$ & $170$ & $280$ & $160$ & $1160$ \\
\tsB{}         & $160$ & $160$ & $200$ & $160$ & \ok{$150$} & $570$ \\
\tsC{}         & $330$ & $330$ & $360$ & $560$ & $310$ & $570$ \\
\oursjl{} (ours), full-gradient uploads & $80$ & \ok{$120$} & \ok{$120$} & \ok{$140$} & \nr & \ok{$110$} \\
\bottomrule
\end{tabular}
\end{table}

Finally, Table~\ref{tab:baselines} compares (a) the best test accuracy and (b) the number of full-gradient uploads $G$ required to reach $60\%$ accuracy for Byzantine-robust baselines without compression and for our algorithm with a compression rate of $50\%$ under several attacks.
The results show that \textsc{GradBant} requires fewer full-gradient transmissions than the remaining baselines to reach $60\%$ accuracy.
We note that under the \textsc{ALIE} attack our method does not reach this level of accuracy, plateauing at $53.9\%$.
In terms of best accuracy, our method outperforms all baselines under the \textsc{IPM} attack.
In the attack-free scenario and under \textsc{SignFlip} and \textsc{LabelFlip} our method stays close to the strongest baselines.

\subsection{Heterogenious scenario}

In this section, we consider a scenario with heterogeneous data distribution across clients. To model heterogeneity, we use a Dirichlet($\alpha$) distribution to partition the dataset samples among clients. This makes the clients heterogeneous both in local dataset size and in the label distribution of their samples. Table~\ref{tab:alpha} reports the training results of our algorithm under four heterogeneity scenarios, both without any attack and under three different attacks.
\begin{table}[h!]
\centering
\caption{Best test accuracy of \oursjl{} ($k/n = 0.5$) under heterogeneous data split by a Dirichlet distribution with concentration $\alpha$ (smaller $\alpha$ means stronger heterogeneity). Byzantine fractions: $60\%$ for label and sign flipping, $70\%$ for IPM.}
\label{tab:alpha}
\small
\setlength{\tabcolsep}{5pt}
\begin{tabular}{l rrrr}
\toprule
Attack & IID & $\alpha=1.0$ & $\alpha=0.5$ & $\alpha=0.2$ \\
\midrule
No attack      & $0.910$ & $0.896$ & $0.872$ & $0.842$ \\
Label flipping & $0.862$ & $0.857$ & $0.782$ & $0.742$ \\
Sign flipping  & $0.878$ & $0.871$ & $0.832$ & $0.753$ \\
IPM            & $0.827$ & $0.826$ & $0.699$ & $0.669$ \\
\bottomrule
\end{tabular}
\end{table}

All results are reported at a 50\% compression rate. As expected, the final model quality degrades as the degree of heterogeneity increases. However, our algorithm remains highly robust to it. In all scenarios, it loses at most $19\%$ accuracy in the worst case and $7\%$ on average relative to the uniform (IID) data distribution.

\section{Conclusion}

We studied Johnson--Lindenstrauss transforms as communication-efficient compressors for Byzantine-robust distributed optimization. Unlike standard sparsification, JL transforms preserve Euclidean geometry with high probability, which suits methods that weight workers by gradient alignment. For a JL-compressed variant of \texttt{Grad-BANT}, we proved a high-probability convergence guarantee under smoothness, average data similarity, and an explicit accuracy condition on the trusted reference gradient, without convexity or a small relative heterogeneity parameter.

We also studied structured, sparse, per-layer, and Kronecker-structured sketches, which reduce the cost of storing and applying dense projections in large models; their moment bounds support the expectation analysis, while the high-probability guarantee additionally requires geometric concentration bounds. The experiments confirm that JL-based compression reduces communication while maintaining Byzantine robustness, suggesting that preserving geometric structure is an important design principle for compression in robust distributed learning.

\end{mainpart}

\begin{appendixpart}

\section{Experimental Details}
\label{app:exps}
\subsection{Protocol.}
\paragraph{Hardware.} All experiments were conducted on a single NVIDIA A100 GPU with 80\,GB of memory. Running all experiments required approximately $200$ GPU-hours in total.

\paragraph{Hyperparameters.} The hyperparameters of each method were tuned via grid search. For each method, all tables report the results obtained with the configuration that achieved the highest validation accuracy.

\begin{table}[h!]
\centering
\caption{Byzantine attacks used in the experiments. Here $g_i$ denotes the stochastic gradient that an honest client $i$ would send, $\bar{g}$ and $\sigma$ denote the coordinate-wise mean and standard deviation of the honest clients' gradients, and $C$ is the number of classes.}
\label{tab:attacks}
\small
\begin{tabularx}{\linewidth}{l X c}
\toprule
Attack & Byzantine behavior & Byz.\ fraction \\
\midrule
Label flipping (LF) & Data poisoning: Byzantine clients compute gradients on their local data with each label $y$ replaced by $C - 1 - y$. & $40\%$, $60\%$ \\
\addlinespace
Sign flipping (SF) & Byzantine clients send the negated honest gradient $-g_i$. & $60\%$ \\
\addlinespace
Random gradients (RG) & Byzantine clients send random vectors with i.i.d.\ Gaussian entries, $\xi \sim \mathcal{N}(0, s^2 I)$. & $40\%$, $60\%$ \\
\addlinespace
ALIE & Omniscient attack: Byzantine clients send $\bar{g} - z\,\sigma$, a small coordinate-wise perturbation of the honest mean that stays within the spread of honest gradients and thus evades robust aggregation rules. & $60\%$ \\
\addlinespace
IPM & Omniscient attack: Byzantine clients send $-\epsilon\,\bar{g}$ with $\epsilon > 0$, so that the inner product between the aggregated update and the true gradient becomes negative. & $70\%$ \\
\bottomrule
\end{tabularx}
\end{table}
\subsection{Trial ablation}
In this section, we study how the size and distribution of the trust set affect the performance of Algorithm~\ref{alg:grad_bant}. All experiments use $k/n = 0.5$, label flipping with $60\%$ Byzantine clients, IID data, and $100$ training rounds. The results are summarized in Table~\ref{tab:trust}.

\paragraph{Trust set size.}
Table~\ref{tab:trust} (a) shows that the size of the trust set has little effect on either the best test accuracy or the convergence speed. As the size increases from $100$ to $2000$ images, the best accuracy varies only between $0.838$ and $0.844$, and the number of rounds needed to reach $60\%$ and $80\%$ accuracy changes by at most $6$ and $5$ rounds, respectively.

\paragraph{Trust set distribution.}
Table~\ref{tab:trust} (b) considers a trust set of $1500$ images whose distribution differs from uniform. Shifts in image appearance slow down training but do not prevent convergence. With grayscale images, reaching $60\%$ and $80\%$ accuracy takes $6$ and $13$ more rounds, respectively, while the best accuracy decreases only slightly, from $0.840$ to $0.831$. Strong augmentation has a larger effect. Reaching $60\%$ accuracy takes almost three times as many rounds, $80\%$ is never reached, and the best accuracy is $0.761$. In both cases, training remains stable, and the final accuracy coincides with the best one. In contrast, class coverage is critical. When the trust set contains only $5$ of the $10$ classes, the best accuracy remains high ($0.828$), but training becomes unstable and the final accuracy drops to $0.636$. With only $2$ classes, the method fails completely: the accuracy never exceeds $0.16$, which is close to random guessing.

\begin{table}[h!]
\centering
\caption{Effect of the trust set on \oursjl{}: (a)~size; (b)~distribution shift, $1500$ images. R$\to 60$ and R$\to 80$: rounds to reach $60\%$ and $80\%$ test accuracy; Best and Final: best and last-round accuracy; \nr{}: not reached. Label flipping, $60\%$ Byzantine clients, $k/n = 0.5$, IID data, $100$ rounds.}
\label{tab:trust}
\small
\setlength{\tabcolsep}{5pt}
\begin{tabular}{l r r r r}
\toprule
Trust set & R$\to 60$ & R$\to 80$ & Best & Final \\
\midrule
\multicolumn{5}{l}{\textit{(a) size, all $10$ classes}} \\
$100$  & \ok{$21$} & \ok{$61$} & \ok{$0.844$} & $0.683$ \\
$500$  & \ok{$21$} & $65$ & $0.841$ & $0.674$ \\
$1000$ & $23$ & $64$ & $0.839$ & $0.622$ \\
$1500$ & $22$ & $65$ & $0.840$ & $0.835$ \\
$2000$ & $27$ & $66$ & $0.838$ & \ok{$0.838$} \\
\midrule
\multicolumn{5}{l}{\textit{(b) shift, $1500$ images}} \\
none                  & \ok{$22$} & \ok{$65$} & \ok{$0.840$} & \ok{$0.835$} \\
grayscale             & $28$ & $78$ & $0.831$ & $0.831$ \\
strong augmentation   & $63$ & \nr  & $0.761$ & $0.761$ \\
$5$ of $10$ classes   & $23$ & $71$ & $0.828$ & $0.636$ \\
$2$ of $10$ classes   & \nr  & \nr  & $0.160$ & $0.102$ \\
\bottomrule
\end{tabular}
\end{table}

Overall, Algorithm~\ref{alg:grad_bant} is robust to the size of the trust set and to moderate shifts in its appearance. The key requirement is that the trust set covers all classes; its size mainly affects the stability of training in later rounds.

\paragraph{Training on the trust set alone.}
To check whether federated training is necessary when a trust set is available, we train ResNet-18 centrally on the trust set alone, using the same optimizer (Adam, learning rate $0.003$), batch size ($64$), and augmentation as the clients. The budget is $12\,400$ optimizer steps, which matches the local computation of one client over $200$ rounds. Table~\ref{tab:reference-only} shows that this baseline is far behind federated training. With $1500$ images, the best accuracy is $0.609$, whereas Algorithm~\ref{alg:grad_bant} with a trust set of the same size reaches $0.913$ within $100$ rounds, even with $60\%$ Byzantine clients (Table~\ref{tab:trust}). The gap grows as the trust set shrinks: with $100$ images, training on the trust set alone reaches at most $0.287$, while the federated method still reaches $0.894$. Even with $2000$ images, the best accuracy is only $0.653$. Moreover, the accuracy saturates early: between the $100$th and $200$th round-equivalents, it changes by less than $0.02$ for all sizes, and the final accuracy is below the best one, which suggests overfitting to the small sample.

\begin{table}[h!]
\centering
\caption{Training ResNet-18 on the trust set alone, without federation. @$r$: CIFAR-10 test accuracy after $62r$ optimizer steps, i.e., the computation of one client over $r$ rounds. Epochs: passes over the trust set. Mean $\pm$ std over $3$ seeds; each seed samples a different trust set. Federated training with the same trust set sizes reaches $0.89$--$0.92$ (Table~\ref{tab:trust}\,(a)).}
\label{tab:reference-only}
\small
\setlength{\tabcolsep}{4pt}
\begin{tabular}{r r r r r r r}
\toprule
Trust set & Epochs & @20 & @50 & @100 & @200 & Best \\
\midrule
$100$  & $7936$ & $0.258\pm0.019$ & $0.260\pm0.014$ & $0.280\pm0.012$ & $0.270\pm0.015$ & $0.287\pm0.010$ \\
$500$  & $1587$ & $0.387\pm0.021$ & $0.407\pm0.027$ & $0.410\pm0.046$ & $0.428\pm0.024$ & $0.439\pm0.023$ \\
$1000$ & $794$  & $0.466\pm0.015$ & $0.496\pm0.011$ & $0.515\pm0.016$ & $0.501\pm0.003$ & $0.542\pm0.012$ \\
$1500$ & $529$  & $0.523\pm0.017$ & $0.554\pm0.021$ & $0.577\pm0.015$ & $0.578\pm0.008$ & $0.609\pm0.006$ \\
$2000$ & $397$  & $0.554\pm0.015$ & $0.607\pm0.012$ & $0.615\pm0.012$ & $0.628\pm0.011$ & $0.653\pm0.004$ \\
\bottomrule
\end{tabular}
\end{table}

\subsection{Additional experiments}
Table~\ref{tab:cp-real-delta} extends the comparison of Table~\ref{tab:adaptive-eps} to compression rates $k/n \in \{0.5, 0.3, 0.1\}$ in the same setting: the distortion is measured between the stochastic ResNet-18 gradient and the trust-set gradient over $100$ pairs.

As expected, the inner-product distortion $\Delta_{\mathrm{ip}}$ increases monotonically for all compressors as the compression becomes stronger. Reducing $k/n$ from $0.5$ to $0.1$ increases $\Delta_{\mathrm{ip}}$ by a factor of $3.0$--$5.2$ for Kronecker JL and of about $5$ for layer-wise SRHT and SRFT.

Kronecker JL remains competitive with layer-wise SRHT and SRFT at all compression rates: at every $k/n$, one of its variants attains the smallest mean $\Delta_{\mathrm{ip}}$ among all compressors, although the differences are within one standard deviation. The angle errors of all JL-type methods are of the same order, ranging from $0.55^\circ$ at $k/n = 0.5$ to $4.2^\circ$ at $k/n = 0.1$.

Applying SRHT and SRFT to the whole parameter vector does not reduce the distortion. At $k/n = 0.5$, it increases $\Delta_{\mathrm{ip}}$ by a factor of $2.9$ for SRHT and $3.9$ for SRFT; at smaller $k/n$, the whole-vector variants are comparable to the layer-wise ones or worse by up to a factor of $1.5$.

Rand-$k$ yields the largest $\Delta_{\mathrm{ip}}$ at every compression rate, $4$--$38$ times larger than that of the other compressors, and its distortion grows the fastest as $k/n$ decreases. At the same time, its angle error remains comparable to that of the other methods.

Finally, Adaptive JL with target $\varepsilon = 0.3$ selects $k/n = 0.240$ and attains $\Delta_{\mathrm{ip}}$ comparable to that of Gaussian Kronecker JL at $k/n = 0.5$ ($96$ vs.\ $94 \cdot 10^{-4}$), i.e., with less than half the message size.

\begin{table}[h!]
\centering
\caption{Distortion of inner products ($\Delta_{\mathrm{ip}}$) and angles ($\Delta_\theta$) under compression at different compression rates: stochastic ResNet-18 gradient vs.\ trust-set gradient, $100$ pairs, mean $\pm$ std. The rows for $k/n = 0.5$ and for Adaptive JL are repeated from Table~\ref{tab:adaptive-eps}. For Adaptive JL, $k/n = 0.240$ is set implicitly by the target $\varepsilon = 0.3$ from~\eqref{eq:en_aware}.}
\label{tab:cp-real-delta}
\small
\setlength{\tabcolsep}{5pt}
\begin{tabular}{l l r r}
\toprule
Compressor & Variant & $\Delta_{\mathrm{ip}}$, $10^{-4}$ & $\Delta_\theta$, deg \\
\midrule
\multicolumn{4}{l}{\textit{$k/n = 0.5$}} \\
Kronecker JL (ours) & Rad.   & $194\pm156$  & $1.1\pm1.7$ \\
                    & Gauss. & $94\pm86$    & $0.8\pm0.7$ \\
SRHT                & layer  & $103\pm80$   & $0.55\pm0.43$ \\
                    & whole  & $302\pm251$  & $1.31\pm1.23$ \\
SRFT                & layer  & $120\pm125$  & $0.63\pm0.58$ \\
                    & whole  & $470\pm198$  & $2.1\pm1.11$ \\
Rand-$k$            &        & $2000\pm800$ & $1.4\pm1.1$ \\
\midrule
\multicolumn{4}{l}{\textit{$k/n = 0.3$}} \\
Kronecker JL (ours) & Rad.   & $300\pm198$ & $1.63\pm0.98$ \\
                    & Gauss. & $390\pm250$ & $1.69\pm1.35$ \\
SRHT                & layer  & $474\pm283$ & $1.72\pm0.96$ \\
                    & whole  & $456\pm341$ & $2.63\pm2.11$ \\
SRFT                & layer  & $534\pm313$ & $2.13\pm2.09$ \\
                    & whole  & $632\pm412$ & $1.87\pm1.31$ \\
Rand-$k$            &        & $4963\pm914$ & $1.57\pm1.16$ \\
\midrule
\multicolumn{4}{l}{\textit{$k/n = 0.1$}} \\
Kronecker JL (ours) & Rad.   & $590\pm490$ & $3.51\pm2.93$ \\
                    & Gauss. & $490\pm490$ & $4.17\pm3.92$ \\
SRHT                & layer  & $512\pm481$ & $2.1\pm1.4$ \\
                    & whole  & $715\pm653$ & $4.15\pm3.31$ \\
SRFT                & layer  & $614\pm541$ & $3.04\pm2.98$ \\
                    & whole  & $891\pm753$ & $3.1\pm2.9$ \\
Rand-$k$            &        & $18500\pm7900$ & $3.3\pm2.4$ \\
\midrule
\multicolumn{4}{l}{\textit{adaptive $k/n$}} \\
Adaptive JL & target $\varepsilon = 0.3$ & $96\pm77$ & $0.51\pm0.43$ \\
\bottomrule
\end{tabular}
\end{table}

Table~\ref{tab:traffic} shows that stronger compression substantially reduces the communication required to reach $80\%$ test accuracy. Moderate compression does not slow down training: in all settings, the number of rounds needed to reach $80\%$ is smallest at $k/n = 0.5$. At $k/n = 0.1$, more rounds are required, but each round is much cheaper, so the total traffic decreases by a factor of $2.8$--$4.0$ compared to $k/n = 0.5$ and by a factor of $6.7$--$10.6$ compared to $k/n = 0.9$. At the same time, the best accuracy at $k/n = 0.1$ is at most $0.018$ lower than at $k/n = 0.5$.

\begin{table}[h]
\centering
\caption{Communication efficiency of \oursjl{} for different compression rates $k/n$ (IID data, $200$ rounds, single seed): (a)~rounds and (b)~traffic in full-gradient uploads (a round costs $k/n$ units) to reach $80\%$ test accuracy; (c)~best test accuracy. Columns $0.9$--$0.5$ use Rademacher JL, columns $0.3$--$0.1$ use Gaussian JL.}
\label{tab:traffic}
\small
\setlength{\tabcolsep}{5pt}
\begin{tabular}{l rrrrr}
\toprule
$k/n$ & $0.9$ & $0.7$ & $0.5$ & $0.3$ & $0.1$ \\
\midrule
\multicolumn{6}{l}{\textit{(a) rounds to $80\%$}} \\
No attack               & $54$ & $45$ & $41$ & $45$ & $73$ \\
Label flipping $60\%$   & $71$ & $62$ & $57$ & $60$ & $71$ \\
Sign flipping $60\%$    & $68$ & $47$ & $46$ & $61$ & $58$ \\
IPM $70\%$              & $68$ & $56$ & $48$ & $51$ & $68$ \\
\midrule
\multicolumn{6}{l}{\textit{(b) traffic to $80\%$, full-gradient uploads}} \\
No attack               & $48.6$ & $31.5$ & $20.5$ & $13.5$ & \ok{$7.3$} \\
Label flipping $60\%$   & $63.9$ & $43.4$ & $28.5$ & $18.0$ & \ok{$7.1$} \\
Sign flipping $60\%$    & $61.2$ & $32.9$ & $23.0$ & $18.3$ & \ok{$5.8$} \\
IPM $70\%$              & $61.2$ & $39.2$ & $24.0$ & $15.3$ & \ok{$6.8$} \\
\midrule
\multicolumn{6}{l}{\textit{(c) best accuracy in $200$ rounds}} \\
No attack               & $0.900$ & $0.906$ & $0.910$ & $0.906$ & $0.893$ \\
Label flipping $60\%$   & $0.853$ & $0.868$ & $0.862$ & $0.865$ & $0.857$ \\
Sign flipping $60\%$    & $0.857$ & $0.875$ & $0.878$ & $0.868$ & $0.860$ \\
IPM $70\%$              & $0.854$ & $0.875$ & $0.827$ & $0.870$ & $0.861$ \\
\bottomrule
\end{tabular}
\end{table}

\label{app:additional_experiments}

\clearpage
\section{Technical Statements}\label{appendix:basic}

Statement below is crucial for deriving high-probability convergence of proposed methods
\begin{lemma}
    Suppose we have events $E_1, E_2, \ldots, E_n$ that holds with following probabilities:
    \[\PP\left[E_i\right] \geq 1-\delta_i,\]
    Then,
    \[\PP\left[\bigcap\limits_{i=1}^n E_i\right] \geq 1 - \sum\limits_{i=1}^n \delta_i\]
    \begin{proof}
        \begin{eqnarray*}
            \PP\left[\bigcap\limits_{i=1}^n E_i\right] = 1 - \PP\left[\bigcup\limits_{i=1}^n \overline{E_i}\right] \geq 1 - \sum\limits_{i=1}^n\PP\left[\overline{E_i} \right] \geq 1 - \sum\limits_{i=1}^n \delta_i
        \end{eqnarray*}
    \end{proof}
\end{lemma}

Next lemma will be useful throughout some theorems, allowing us to deal with norms:
\begin{lemma}[Young's inequality]
    Let $x,y \in \RR^n$. Then, for any $\alpha > 0$ we have
    \[\<x,y> \leq \frac{1}{2\alpha}\|x\|^2 + \frac{\alpha}{2}\|y\|^2.\]
\end{lemma}

\clearpage

\section{JL Statements}\label{ap:jl_sec}

We start with the properties of a single sketch: inner-product preservation,
moments of the lifted vector, and high-probability norm bounds. We then
show how these properties extend to structured matrices, separate layers,
and tensor modes. Throughout, $Sx$ is the compressed vector and $S^\top Sx$
is the lifted vector used in the update. We write $\|\cdot\|$ for the
Euclidean norm of a vector or the operator norm of a matrix, and
$\|\cdot\|_F$ for the Frobenius norm of a matrix or tensor.

\subsection{From norm preservation to inner products}

\begin{lemma}[Preservation of scalar product]\label{lem:app_inner_product}
Let $S$ be a linear JL transform with distortion $\e$ and one-vector failure
probability $\delta$. For every fixed $x,y\in\RR^n$,
\[
\PP\left[
\left|\langle Sx,Sy\rangle-\langle x,y\rangle\right|
\le \e\|x\|\|y\|
\right]\ge 1-2\delta.
\]
\end{lemma}
\begin{proof}
The claim is immediate if either vector is zero. By linearity, it is enough
to prove it for unit vectors. Apply the JL property to $x+y$ and $x-y$.
With probability at least $1-2\delta$, both squared norms are preserved
within relative error $\e$. On this event, polarization and the
parallelogram identity give
\[
\begin{split}
\left|\langle Sx,Sy\rangle-\langle x,y\rangle\right|
&\le \frac{\e}{4}\left(\|x+y\|^2+\|x-y\|^2\right)\\
&=\frac{\e}{2}\left(\|x\|^2+\|y\|^2\right)
=\e.
\end{split}
\]
For general nonzero vectors, apply this bound to $x/\|x\|$ and
$y/\|y\|$, and multiply by $\|x\|\|y\|$.
\end{proof}

\begin{lemma}[Simultaneous upper JL bound]\label{c_e}
Let $S$ be a linear JL transform with distortion $\e$ and one-vector failure
probability $\delta$. For fixed vectors $v_1,\dots,v_m\in\RR^n$,
\[
\PP\left[
\|Sv_i\|^2\le(1+\e)\|v_i\|^2
\quad\text{for all }i=1,\dots,m
\right]\ge1-m\delta.
\]
\end{lemma}
\begin{proof}
Each upper bound fails with probability at most $\delta$. Take the union
bound over the $m$ vectors.
\end{proof}

\subsection{Dense sketches: moments and operator norms}

The next identity gives the moments of both dense sketches and the sparse
Gaussian construction below. It also allows an independent orthogonal
transformation before sketching.

\begin{lemma}[Moments of a normalized i.i.d. matrix]\label{lem:app_iid_core}
Let $A\in\RR^{k\times n}$ have i.i.d. entries with
\[
\EE A_{11}=0,\qquad
\EE A_{11}^2=1,\qquad
\EE A_{11}^4=\mu_4<\infty.
\]
Let $U\in\RR^{n\times n}$ be orthogonal and independent of $A$; $U$ may be
deterministic or random. For $S=AU/\sqrt{k}$,
\[
\EE[S^\top S]=I_n,
\qquad
\EE\|S^\top Sx\|^2
=\left(1+\frac{n-2+\mu_4}{k}\right)\|x\|^2
.
\]
The second identity holds for every fixed $x\in\RR^n$.
\end{lemma}
\begin{proof}
First take $U=I_n$, and write the rows of $A$ as column vectors
$a_1,\dots,a_k$. Since $\EE[a_ra_r^\top]=I_n$,
\[
\EE[S^\top S]
=\frac1k\sum_{r=1}^k\EE[a_ra_r^\top]=I_n.
\]
For the second moment, expand
\[
\EE\|A^\top Ax\|^2
=\sum_{r,s=1}^k
\EE\left[
\langle a_r,a_s\rangle\langle a_r,x\rangle\langle a_s,x\rangle
\right].
\]
If $r\ne s$, conditioning on $a_r$ shows that the corresponding expectation
is $\EE\langle a_r,x\rangle^2=\|x\|^2$. For a diagonal term, independence
and centering of the entries remove all mixed terms, giving
\[
\begin{split}
\EE\left[\|a_r\|^2\langle a_r,x\rangle^2\right]
&=\sum_{j=1}^n\left(\sum_{i\ne j}x_i^2+\mu_4x_j^2\right)\\
&=(n-1+\mu_4)\|x\|^2.
\end{split}
\]
There are $k(k-1)$ off-diagonal terms and $k$ diagonal terms. Dividing their
sum by $k^2$ yields
\[
\EE\|S^\top Sx\|^2
=\frac{k(k-1)+k(n-1+\mu_4)}{k^2}\|x\|^2
=\left(1+\frac{n-2+\mu_4}{k}\right)\|x\|^2.
\]
For general $U$, condition on $U$. Orthogonality reduces the calculation
to the case just proved with $x$ replaced by $Ux$, whose norm equals
$\|x\|$. Averaging over $U$ proves both identities.
\end{proof}

\begin{lemma}[Gaussian and Rademacher moments]
\label{lem:app_gaussian_moments}\label{lem:app_rademacher_moments}
Let $S\in\RR^{k\times n}$ have independent entries, either
$S_{ij}\sim\mathcal N(0,1/k)$ or
$S_{ij}\in\{-1/\sqrt{k},1/\sqrt{k}\}$ with equal probabilities. In both
cases $\EE[S^\top S]=I_n$, and for every fixed $x\in\RR^n$,
\[
\EE\|S^\top Sx\|^2
=
\begin{cases}
\displaystyle\frac{n+k+1}{k}\|x\|^2,&\text{Gaussian},\\[2mm]
\displaystyle\frac{n+k-1}{k}\|x\|^2,&\text{Rademacher}.
\end{cases}
\]
\end{lemma}
\begin{proof}
Apply Lemma~\ref{lem:app_iid_core} with $U=I_n$. The fourth moment of a
standard Gaussian entry is $3$, and that of a Rademacher entry is $1$.
\end{proof}

We next bound the operator norm. Every such bound also controls the
reconstructed vector, since for every matrix $S$ and every $x$,
\[
\|S^\top S\|=\|S\|^2,
\qquad
\|S^\top Sx\|^2\le\|S\|^4\|x\|^2.
\]
The following standard estimate will be used for the Rademacher and sparse
Gaussian constructions. Its subgaussian norm is defined, for a scalar $X$
and a random vector $Z$, by
\[
\|X\|_{\psi_2}
:=\inf\{s>0:\EE e^{X^2/s^2}\le2\},
\qquad
\|Z\|_{\psi_2}
:=\sup_{\|u\|=1}\|\langle Z,u\rangle\|_{\psi_2}.
\]

\begin{lemma}[Subgaussian operator-norm estimate \citep{Vershynin_2026}]
\label{lem:app_subgaussian_operator_norm}
Let $A\in\RR^{m\times n}$ have independent, mean-zero, isotropic,
subgaussian rows $A_1,\dots,A_m$, with $\|A_i\|_{\psi_2}\le K$ for all
$i$. There is a universal constant $C>0$ such that, for every $t\ge0$,
\[
\PP\left[
\|A\|\le\sqrt m+CK^2(\sqrt n+t)
\right]\ge1-2e^{-t^2}.
\]
\end{lemma}

\begin{lemma}[Gaussian and Rademacher operator norms]
\label{lem:app_gaussian_operator_norm}\label{lem:app_rademacher_operator_norm}
There is a universal constant $C>0$ such that, for a Gaussian or a
Rademacher JL matrix $S\in\RR^{k\times n}$ and every $\delta\in(0,1)$,
with probability at least $1-\delta$,
\[
\|S\|\le
1+C\left(\sqrt{\frac nk}+\sqrt{\frac{\log(2/\delta)}{k}}\right).
\]
Squaring the right-hand side bounds $\|S^\top S\|$; its fourth power
times $\|x\|^2$ bounds $\|S^\top Sx\|^2$ simultaneously for all
$x\in\RR^n$.
\end{lemma}
\begin{proof}
In the Gaussian case, write $S=A/\sqrt{k}$ with standard Gaussian entries
in $A$. By \citet[Corollary~7.3.2]{Vershynin_2026}, there is a universal
constant $c_0>0$ such that, for every $t\ge0$,
\[
\PP\left[\|A\|\ge\sqrt k+\sqrt n+t\right]\le 2e^{-c_0t^2}.
\]
Taking $t=\sqrt{\log(2/\delta)/c_0}$ and dividing by $\sqrt k$ gives, with
probability at least $1-\delta$,
\[
\|S\|\le1+\sqrt{\frac nk}+\frac{1}{\sqrt{c_0}}\sqrt{\frac{\log(2/\delta)}{k}},
\]
which is the stated bound for any $C\ge\max\{1,c_0^{-1/2}\}$.

In the Rademacher case, write again $S=A/\sqrt{k}$. Its independent rows
are mean-zero and isotropic. For each unit vector $u$, Hoeffding's lemma
gives
\[
\EE e^{\lambda\langle A_i,u\rangle}\le e^{\lambda^2/2}
\quad\text{for all }\lambda\in\RR,
\]
so their subgaussian norms are bounded by a universal constant. Apply
Lemma~\ref{lem:app_subgaussian_operator_norm} with $m=k$ and
$t=\sqrt{\log(2/\delta)}$, and divide by $\sqrt{k}$. Enlarging $C$ if
necessary, one constant serves both constructions. The remaining claims
follow from $\|S^\top S\|=\|S\|^2$ and
$\|S^\top Sx\|\le\|S^\top S\|\|x\|$.
\end{proof}

\subsection{Preserving the total squared norm of several vectors}

The convergence proof needs a bound on the sum of the squared norms of the
compressed honest gradients. Applying a separate JL guarantee to every
gradient would introduce a failure-probability factor equal to the number
of honest workers. For Gaussian and Rademacher sketches, we can instead
control this sum directly.

Place the vectors in the columns of a fixed matrix $V\in\RR^{n\times N}$.
Then $\|V\|_F^2$ is their total squared norm and $\|SV\|_F^2$ is the
corresponding quantity after compression. We first give an elementary
concentration bound used in the Gaussian case.

\begin{lemma}[Weighted chi-square concentration]\label{lem:app_weighted_chi_square}
Let $a_1,\ldots,a_r\geq0$ and let $Z_1,\ldots,Z_r$ be independent
$\chi_k^2$ random variables. For every $\e\in(0,1)$,
\[
    \PP\left[
        \left|\frac1k\sum_{i=1}^r a_iZ_i-\sum_{i=1}^r a_i\right|
        >\e\sum_{i=1}^r a_i
    \right]
    \leq 2e^{-k\e^2/16}.
\]
\begin{proof}
The claim is immediate if all weights vanish. Otherwise, divide the weights
by their sum, so that $\sum_i a_i=1$. For a single collection of independent
standard Gaussian variables $g_i$, the Gaussian moment-generating function
gives, for $0\leq\lambda\leq1/4$,
\[
\begin{aligned}
    \log\EE\exp\left(\lambda\sum_i a_i(g_i^2-1)\right)
    &= -\lambda-\frac12\sum_i\log(1-2\lambda a_i)
      \leq4\lambda^2\sum_i a_i^2\leq4\lambda^2,\\
    \log\EE\exp\left(-\lambda\sum_i a_i(g_i^2-1)\right)
    &= \lambda-\frac12\sum_i\log(1+2\lambda a_i)
      \leq\lambda^2\sum_i a_i^2\leq\lambda^2.
\end{aligned}
\]
The second bound also holds for every $\lambda\geq0$. Write each $Z_i$ as
the sum of $k$ independent squared standard Gaussian variables. Independence
across these $k$ repetitions and the exponential Markov inequality give
upper and lower tail bounds
\[
    e^{-k\lambda\e+4k\lambda^2}
    \quad\text{and}\quad
    e^{-k\lambda\e+k\lambda^2},
\]
respectively. Taking $\lambda=\e/8$ in the first bound and
$\lambda=\e/2$ in the second proves the result.
\end{proof}
\end{lemma}

\begin{lemma}[Preservation of a fixed Frobenius norm]\label{lem:app_one_factor_frobenius}
Let $S\in\RR^{k\times n}$ be a Gaussian or Rademacher JL matrix.
There is a universal constant $c_1>0$ such that, for every fixed matrix
$V\in\RR^{n\times N}$ and every $\e\in(0,1)$,
\[
    \PP\left[
        \left|\|SV\|_F^2-\|V\|_F^2\right|>\e\|V\|_F^2
    \right]
    \leq2e^{-c_1k\e^2}.
\]
In particular, the bound does not depend on the number of columns $N$.
\begin{proof}
The case $V=0$ is immediate.

\emph{Gaussian sketches.}
Take a compact singular value decomposition of $V$, with nonzero singular
values $\sigma_1,\ldots,\sigma_r$. Gaussian rotational invariance gives
\[
    \|SV\|_F^2\overset{d}{=}
    \frac1k\sum_{i=1}^r\sigma_i^2Z_i,
    \qquad
    \|V\|_F^2=\sum_{i=1}^r\sigma_i^2,
\]
where the $Z_i$ are independent $\chi_k^2$ variables.
Apply Lemma~\ref{lem:app_weighted_chi_square} with weights $\sigma_i^2$.

\emph{Rademacher sketches.}
Let $\xi_1,\ldots,\xi_k$ be the rows of $\sqrt{k}S$, viewed as column
vectors, and set $B=VV^\top$. Then
\[
    k\|SV\|_F^2=\sum_{\ell=1}^k\xi_\ell^\top B\xi_\ell,
    \qquad
    \EE\|SV\|_F^2=\operatorname{tr}(B)=\|V\|_F^2.
\]
Apply the Hanson--Wright inequality to the concatenated vector of all
$kn$ independent signs and the block-diagonal matrix with $k$ copies of
$B$. For every $s>0$, this gives
\[
    \PP\left[
        \left|k\|SV\|_F^2-k\|V\|_F^2\right|>s
    \right]
    \leq2\exp\left(
        -c\min\left\{\frac{s^2}{k\|B\|_F^2},
                          \frac{s}{\|B\|}\right\}
    \right).
\]
Since $B$ is positive semidefinite, both $\|B\|$ and $\|B\|_F$ are at
most $\operatorname{tr}(B)=\|V\|_F^2$. With
$s=k\e\|V\|_F^2$ and $0<\e<1$, the minimum is at least $k\e^2$.
Choosing a common constant for the two constructions completes the proof.
\end{proof}
\end{lemma}

This estimate concerns one fixed matrix $V$, rather than all matrices
simultaneously. It also applies conditionally when $V$ is determined by the
past and $S$ is sampled independently afterwards. Together with an
operator-norm bound, it controls the lifted vectors through
\[
    \|S^\top SV\|_F^2\leq\|S\|^2\|SV\|_F^2.
\]
The same Frobenius-norm estimate will allow us to compress a tensor one
mode at a time below.

\subsection{Structured sketches}

The following results give the same moment and operator-norm information
for SRHT, the FJLT with a sparse Gaussian core, and CountSketch.

\begin{lemma}[SRHT]\label{lem:app_srht}
Assume that $n$ is a power of two and $1\le k\le n$. Let $H$ be the
normalized Walsh--Hadamard matrix, let $D$ have independent Rademacher
entries on its diagonal, and let $T\subset[n]$ be a uniformly random
$k$-element subset independent of $D$. If $R_T$ restricts a vector to its
coordinates in $T$, define
\[
S=\sqrt{\frac nk}\,R_THD.
\]
Then
\[
\EE[S^\top S]=I_n,
\qquad
\EE\|S^\top Sx\|^2=\frac nk\|x\|^2
.
\]
The second identity holds for every fixed $x\in\RR^n$. Moreover,
$\|S\|=\sqrt{n/k}$ almost surely.
\end{lemma}
\begin{proof}
The matrix $R_T^\top R_T$ is the coordinate projector onto $T$. Each
coordinate is selected with probability $k/n$, so
\[
\EE[R_T^\top R_T]=\frac knI_n.
\]
Since $HD$ is orthogonal, conditioning on $D$ gives
\[
\EE[S^\top S\mid D]
=\frac nkD H^\top\EE[R_T^\top R_T]HD=I_n.
\]
Also, $R_TR_T^\top=I_k$ implies
\[
SS^\top=\frac nkI_k,
\qquad
(S^\top S)^2=\frac nkS^\top S.
\]
The first identity gives $\|S\|=\sqrt{n/k}$. The second and unbiasedness
give
\[
\EE\|S^\top Sx\|^2
=x^\top\EE[(S^\top S)^2]x
=\frac nk\|x\|^2.
\]
\end{proof}

\begin{lemma}[FJLT with a sparse Gaussian core]\label{lem:app_fjlt}
Assume that $n$ is a power of two. Let $H$ be the normalized
Walsh--Hadamard matrix, let $D$ be a diagonal Rademacher matrix, and let
$P\in\RR^{k\times n}$ have independent entries
\[
P_{ij}=
\begin{cases}
0,&\text{with probability }1-q,\\
g_{ij}/\sqrt q,&\text{with probability }q,
\end{cases}
\qquad g_{ij}\sim\mathcal N(0,1),\qquad q\in(0,1].
\]
Assume that all random choices are independent, and define
$S=PHD/\sqrt{k}$. Then
\[
\EE[S^\top S]=I_n,
\qquad
\EE\|S^\top Sx\|^2
=\left(1+\frac{n-2+3/q}{k}\right)\|x\|^2
\quad\text{for every fixed }x\in\RR^n.
\]
Moreover, for every $\delta\in(0,1)$, with probability at least $1-\delta$,
\[
\|S\|\le
1+\frac Cq\left(\sqrt{\frac nk}+\sqrt{\frac{\log(2/\delta)}{k}}\right),
\]
where $C>0$ is universal.
\end{lemma}
\begin{proof}
The entries of $P$ satisfy
\[
\EE P_{11}=0,\qquad \EE P_{11}^2=1,\qquad \EE P_{11}^4=3/q.
\]
The moment identities follow from Lemma~\ref{lem:app_iid_core} with
$A=P$ and the independent orthogonal matrix $U=HD$.

To bound the operator norm, observe that
\[
\EE e^{\lambda P_{11}}
=1-q+q e^{\lambda^2/(2q)}\le e^{\lambda^2/(2q)}.
\]
By independence, for a row $p$ of $P$ and a unit vector $u$,
\[
\EE e^{\lambda\langle p,u\rangle}
\le\prod_{j=1}^n e^{\lambda^2u_j^2/(2q)}
=e^{\lambda^2/(2q)}.
\]
Thus the rows are mean-zero and isotropic, with subgaussian norms at most
$C/\sqrt q$. Lemma~\ref{lem:app_subgaussian_operator_norm}, with
$t=\sqrt{\log(2/\delta)}$, gives
\[
\|P\|\le\sqrt k+\frac Cq\left(\sqrt n+\sqrt{\log(2/\delta)}\right)
\]
with the claimed probability, after changing the universal constant.
Finally, orthogonality of $HD$ gives $\|S\|=\|P\|/\sqrt{k}$.
\end{proof}

\begin{lemma}[CountSketch]\label{lem:app_countsketch}
Let $h(1),\dots,h(n)$ be independent and uniform on $[k]$, and let
$\sigma(1),\dots,\sigma(n)$ be independent Rademacher signs, independent
of $h$. Define $S\in\RR^{k\times n}$ by
\[
S_{rj}=\sigma(j)\mathbf1_{\{h(j)=r\}}.
\]
Then
\[
\EE[S^\top S]=I_n,
\qquad
\EE\|S^\top Sx\|^2
=\left(1+\frac{n-1}{k}\right)\|x\|^2
\quad\text{for every fixed }x\in\RR^n.
\]
If $N_r=\#\{j\in[n]:h(j)=r\}$ is the number of coordinates assigned to
row $r$, then
\[
SS^\top=\operatorname{diag}(N_1,\dots,N_k),
\qquad
\|S\|^2=\max_{1\le r\le k}N_r.
\]
Consequently, for every $\delta\in(0,1)$, with probability at least
$1-\delta$,
\[
\|S\|^2\le
\frac nk+\sqrt{\frac{2n}{k}\log\frac{k}{\delta}}
+\frac23\log\frac{k}{\delta}.
\]
\end{lemma}
\begin{proof}
Every column contains one nonzero entry of magnitude one. Hence
$(S^\top S)_{jj}=1$, while for $i\ne j$,
\[
(S^\top S)_{ij}
=\sigma(i)\sigma(j)\mathbf1_{\{h(i)=h(j)\}}.
\]
Independent centered signs make the off-diagonal expectations zero,
proving unbiasedness. To compute the second moment, write
\[
(S^\top Sx-x)_j
=\sigma(j)\sum_{i\ne j}\sigma(i)\mathbf1_{\{h(i)=h(j)\}}x_i.
\]
After squaring, the mixed terms vanish in expectation. Since each pair of
coordinates collides with probability $1/k$,
\[
\EE\|S^\top Sx-x\|^2
=\frac1k\sum_{j=1}^n\sum_{i\ne j}x_i^2
=\frac{n-1}{k}\|x\|^2.
\]
Unbiasedness now gives the claimed second moment.

Distinct rows of $S$ have disjoint supports, and the squared norm of row
$r$ is $N_r$. This proves the formula for $SS^\top$ and its operator norm.
For each $r$, the count $N_r$ is binomial with mean $n/k$. Bernstein's
inequality gives, for every $u\ge0$,
\[
\PP\left[
N_r>\frac nk+\sqrt{\frac{2nu}{k}}+\frac{2u}{3}
\right]\le e^{-u}.
\]
Take $u=\log(k/\delta)$ and apply the union bound over the $k$ rows.
\end{proof}

\subsection{Per-layer sketching}

We now combine the one-layer guaranties using the block-diagonal structure
of the sketch. Write
\[
x=(x_1,\dots,x_W),\qquad
x_w\in\RR^{n_w},\qquad
n=\sum_{w=1}^W n_w,\qquad
S=\operatorname{diag}(S_1,\dots,S_W),
\]
where $S_w\in\RR^{k_w\times n_w}$. The key identities are
\[
\|x\|^2=\sum_{w=1}^W\|x_w\|^2,
\qquad \|Sx\|^2=\sum_{w=1}^W\|S_wx_w\|^2.
\]
\[
S^\top S=\operatorname{diag}(S_1^\top S_1,\dots,S_W^\top S_W),
\qquad
\|S^\top Sx\|^2=\sum_{w=1}^W\|S_w^\top S_wx_w\|^2.
\]
As in the main text, for $x\ne0$ we use the energy fractions
\[
\alpha_w(x)=\frac{\|x_w\|^2}{\|x\|^2},
\qquad
\sum_{w=1}^W\alpha_w(x)=1.
\]
Thus, squared-norm bounds combine as energy-weighted averages, whereas
operator-norm bounds combine by taking the largest layer bound.

\begin{theorem}[Per-layer JL properties]\label{thm:app_per_layer_jl}
Suppose that, for each $w$, there are $\e_w\geq0$ and
$\delta_w\in[0,1]$ such that, for every fixed $u\in\RR^{n_w}$,
\[
    \PP\left[
        \left|\|S_wu\|^2-\|u\|^2\right|>\e_w\|u\|^2
    \right]\leq\delta_w.
\]
Then, for every fixed
$x\ne0$,
\[
\PP\left[
\left|\|Sx\|^2-\|x\|^2\right|
>
\left(\sum_{w=1}^W\alpha_w(x)\e_w\right)\|x\|^2
\right]
\le\sum_{w=1}^W\delta_w.
\]
In particular, if $\max_w\e_w\le\e$, then $S$ satisfies the
$(\e,\sum_w\delta_w)$ distributional JL property. If the blocks are
independent, both failure bounds can be replaced by
$1-\prod_{w=1}^W(1-\delta_w)$.
\end{theorem}
\begin{proof}
Fix $x\ne0$. With probability at least $1-\sum_w\delta_w$, every layer
satisfies
$|\|S_wx_w\|^2-\|x_w\|^2|\le\e_w\|x_w\|^2$.
On this event,
\[
\left|\|Sx\|^2-\|x\|^2\right|
\le
\sum_{w=1}^W\left|\|S_wx_w\|^2-\|x_w\|^2\right|
\le
\sum_{w=1}^W\e_w\|x_w\|^2
=
\left(\sum_{w=1}^W\alpha_w(x)\e_w\right)\|x\|^2.
\]
The weighted distortion is at most $\max_w\e_w$, since the weights sum
to one. For independent blocks, the layer success events are independent,
so their intersection has probability at least
$\prod_w(1-\delta_w)$.
\end{proof}

The same block identities give the properties of the lifted map
$S^\top S$ needed in the optimization analysis. None of the next three
results requires independence between layers.

\begin{lemma}[Per-layer unbiasedness]\label{lem:app_per_layer_unbiased}
If $\EE[S_w^\top S_w]=I_{n_w}$ for every layer, then
\[
\EE[S^\top S]=I_n,
\qquad
\EE[S^\top Sx]=x\quad\text{for all }x\in\RR^n.
\]
\end{lemma}
\begin{proof}
Take expectations in the block-diagonal identity for $S^\top S$.
\end{proof}

\begin{lemma}[Per-layer second moment]\label{lem:app_per_layer_second_moment}
Suppose, for each layer, there is $\omega_w\ge0$ such that
\[
\EE\|S_w^\top S_wu\|^2\le\omega_w\|u\|^2
\qquad\text{for every }u\in\RR^{n_w}.
\]
Then, for every $x\ne0$,
\[
\EE\|S^\top Sx\|^2
\le
\left(\sum_{w=1}^W\alpha_w(x)\omega_w\right)\|x\|^2
\le
\left(\max_w\omega_w\right)\|x\|^2.
\]
\end{lemma}
\begin{proof}
Using the sum of squared block norms and linearity of expectation gives
\[
\EE\|S^\top Sx\|^2
=\sum_{w=1}^W\EE\|S_w^\top S_wx_w\|^2
\le\sum_{w=1}^W\omega_w\|x_w\|^2.
\]
Divide by $\|x\|^2$ and use $\sum_w\alpha_w(x)=1$.
\end{proof}

\begin{lemma}[Per-layer high-probability norm bound]\label{lem:app_per_layer_norm}
Suppose $\PP[\|S_w\|^2\le c_w]\ge1-\delta_w$ for every layer.
Then, with probability at least $1-\sum_w\delta_w$,
\[
\|S\|^2=\|S^\top S\|\le\max_w c_w,
\]
and, on the same event, for every $x\ne0$,
\[
\begin{aligned}
\|Sx\|^2
&\le\left(\sum_{w=1}^W\alpha_w(x)c_w\right)\|x\|^2,\\
\|S^\top Sx\|^2
&\le\left(\sum_{w=1}^W\alpha_w(x)c_w^2\right)\|x\|^2.
\end{aligned}
\]
\end{lemma}
\begin{proof}
A union bound ensures that all inequalities $\|S_w\|^2\le c_w$ hold
on an event of the stated probability. On this event, block diagonality
gives
\[
\|S\|^2=\|S^\top S\|=\max_w\|S_w\|^2\le\max_w c_w.
\]
For any $x$, summing the layer bounds gives
\[
\|Sx\|^2\le\sum_w c_w\|x_w\|^2,
\qquad
\|S^\top Sx\|^2\le\sum_w\|S_w\|^4\|x_w\|^2
\le\sum_w c_w^2\|x_w\|^2.
\]
Substituting $\|x_w\|^2=\alpha_w(x)\|x\|^2$ proves the claim.
\end{proof}

\paragraph{Allocating the sketch dimensions.}
Theorem~\ref{thm:app_per_layer_jl} shows that a layer containing little
of a vector's energy can tolerate a larger relative distortion. For
Gaussian and Rademacher sketches, this allows a smaller sketch dimension,
since $k_w$ scales as $\e_w^{-2}$. The following allocation uses this
weighted distortion budget for a fixed vector.

\begin{corollary}[Energy-aware allocation]\label{cor:app_energy_aware}
Fix $x\ne0$ and $\e,\delta\in(0,1)$. For every layer with
$\alpha_w(x)>0$, choose a Gaussian or Rademacher sketch with
\[
\e_w=
\frac{\e\,\alpha_w(x)^{-1/3}}
{\sum_{j=1}^W\alpha_j(x)^{2/3}},
\qquad
\delta_w=\frac{\delta}{W},
\qquad
k_w=
\left\lceil C\frac{\log(2W/\delta)}{\e_w^2}\right\rceil,
\]
where $C>0$ is a sufficiently large universal constant. Assume that
$\e_w<1$ and $k_w\le n_w$ for all such layers. Layers with $x_w=0$
are omitted, with $k_w=0$.

Then, for this fixed vector,
\[
\PP\left[\left|\|Sx\|^2-\|x\|^2\right|\le\e\|x\|^2\right]
\ge1-\delta,
\]
and the total number of transmitted coordinates satisfies
\[
\sum_{w=1}^W k_w
=\Theta\left(
\frac{\log(2W/\delta)}{\e^2}
\left(\sum_{w=1}^W\alpha_w(x)^{2/3}\right)^3
\right).
\]
\end{corollary}
\begin{proof}
The one-layer JL guarantees and a union bound give simultaneous success
on all nonzero layers with probability at least $1-\delta$. Zero layers
contribute no error for the fixed vector $x$. The total relative
distortion on this event is at most
\[
\sum_{w:\alpha_w(x)>0}\alpha_w(x)\e_w
=
\e\,
\frac{\sum_{w:\alpha_w(x)>0}\alpha_w(x)^{2/3}}
{\sum_{j=1}^W\alpha_j(x)^{2/3}}
=\e,
\]
as in Theorem~\ref{thm:app_per_layer_jl}.

To count coordinates, substitute the chosen distortions:
\[
\sum_{w:\alpha_w(x)>0}\frac{1}{\e_w^2}
=
\frac{\left(\sum_{j=1}^W\alpha_j(x)^{2/3}\right)^2}{\e^2}
\sum_{w:\alpha_w(x)>0}\alpha_w(x)^{2/3}
=
\frac{1}{\e^2}
\left(\sum_{w=1}^W\alpha_w(x)^{2/3}\right)^3.
\]
Since $0<\e_w<1$ and $\log(2W/\delta)\ge\log2$, rounding each
nonzero layer's dimension changes it by at most a universal constant
factor. Multiplying the last identity by $\log(2W/\delta)$ therefore
gives the stated total dimension.
\end{proof}

The allocation depends on the energy profile of $x$: its guarantee is
for that fixed vector, with these dimensions chosen before the sketches
are drawn. It applies only in the stated feasible regime
$\e_w<1$ and $k_w\le n_w$.

\subsection{Kronecker sketches: compressing one tensor mode at a time}

Per-layer sketches act on separate blocks. A Kronecker sketch instead
applies several smaller matrices to different modes of the same tensor.
We use the notation from the main text: for
$X\in\RR^{n_1\times\cdots\times n_d}$ and
$S_j\in\RR^{k_j\times n_j}$, define
\[
    \mathcal S(X)=X\times_1S_1\cdots\times_dS_d,
    \qquad S=S_d\otimes\cdots\otimes S_1.
\]
Here $X\times_jS_j$ means multiplication along the $j$th mode. More
explicitly, if $X_{(j)}$ is the matrix formed by arranging the mode-$j$
fibers of $X$ as columns, then
\[
    (X\times_jS_j)_{(j)}=S_jX_{(j)},
    \qquad
    \operatorname{vec}(\mathcal S(X))=S\operatorname{vec}(X).
\]
Both unfolding and vectorization preserve the Frobenius norm. Thus a
single mode multiplication can be analyzed using a matrix Frobenius-norm
bound. We prove this composition step once, then apply it to dense
sketches and to general JL factors.

\begin{theorem}[Combining independent Frobenius-norm guarantees]
\label{thm:app_abstract_frobenius_jl}
Let $X$ be a fixed tensor and let $S_1,\ldots,S_d$ be independent random
matrices. Fix $\e\in(0,1)$ and $\delta_1,\ldots,\delta_d\in[0,1]$.
Suppose that, for each $j$ and every fixed matrix $V$ with $n_j$ rows,
\[
    \PP\left[
        \left|\|S_jV\|_F^2-\|V\|_F^2\right|
        >\frac{\e}{4d}\|V\|_F^2
    \right]\leq\delta_j.
\]
Then
\[
    \PP\left[
        \left|\|\mathcal S(X)\|_F^2-\|X\|_F^2\right|
        >\e\|X\|_F^2
    \right]\leq\sum_{j=1}^d\delta_j.
\]
\begin{proof}
Let $X^{(0)}=X$ and $X^{(j)}=X^{(j-1)}\times_jS_j$ denote the tensor
after the first $j$ multiplications. Conditional on $S_1,\ldots,S_{j-1}$,
the tensor $X^{(j-1)}$ is fixed and $S_j$ remains independent. The
assumption applied to $(X^{(j-1)})_{(j)}$ therefore gives
\[
\begin{split}
    \PP\Bigg[
        \left|\|X^{(j)}\|_F^2-\|X^{(j-1)}\|_F^2\right|
        >\frac{\e}{4d}\|X^{(j-1)}\|_F^2
        \;\Bigm|\;S_1,\ldots,S_{j-1}
    \Bigg]\leq\delta_j.
\end{split}
\]
Taking expectations and a union bound, all $d$ steps satisfy their bounds
with probability at least $1-\sum_j\delta_j$. On this event,
\[
    \left(1-\frac{\e}{4d}\right)^d\|X\|_F^2
    \leq\|\mathcal S(X)\|_F^2
    \leq\left(1+\frac{\e}{4d}\right)^d\|X\|_F^2.
\]
For $0<\e<1$, Bernoulli's inequality and $1+u\leq e^u$ give
\[
    \left(1-\frac{\e}{4d}\right)^d\geq1-\frac{\e}{4}\geq1-\e,
    \qquad
    \left(1+\frac{\e}{4d}\right)^d\leq e^{\e/4}\leq1+\e.
\]
This proves the claim. The proof also shows that it is enough to bound
the conditional failure probability at each step; a simultaneous guarantee
over all possible intermediate tensors is not required.
\end{proof}
\end{theorem}

\begin{theorem}[Gaussian and Rademacher Kronecker sketches]
\label{thm:app_kronecker_concrete}
Let the factors $S_j\in\RR^{k_j\times n_j}$ be independent Gaussian or
Rademacher JL matrices. There is a universal constant $c>0$ such that,
for every fixed tensor $X$ and every $\e\in(0,1)$,
\[
    \PP\left[
        \left|\|\mathcal S(X)\|_F^2-\|X\|_F^2\right|
        >\e\|X\|_F^2
    \right]
    \leq2\sum_{j=1}^d e^{-c\e^2k_j/d^2}
    \leq2d\,e^{-c\e^2\min_j k_j/d^2}.
\]
Consequently, failure probability at most $\delta$ is ensured by
\[
    \min_{1\leq j\leq d}k_j
    \geq C\frac{d^2}{\e^2}\log\frac{2d}{\delta},
\]
where $C>0$ is a sufficiently large universal constant. The same bounds
hold for $\|Sx\|^2$ with $x=\operatorname{vec}(X)$.
\begin{proof}
Lemma~\ref{lem:app_one_factor_frobenius}, with distortion $\e/(4d)$,
bounds the failure probability at step $j$ by
$2e^{-c_1k_j\e^2/(16d^2)}$. Apply the conditioning argument of
Theorem~\ref{thm:app_abstract_frobenius_jl} and set $c=c_1/16$.
The dimension bound follows by requiring each of the $d$ terms to be at
most $\delta/d$. Vectorization gives the final assertion.
\end{proof}
\end{theorem}

\paragraph{When only a fixed-vector JL guarantee is available.}
The preceding dense-sketch result uses the Frobenius-norm estimate of
Lemma~\ref{lem:app_one_factor_frobenius}, whose failure probability is
independent of the matrix rank. For another construction, we can obtain
a Frobenius-norm guarantee from fixed-vector bounds, at the cost of a
rank factor.

\begin{lemma}[From fixed vectors to a fixed matrix]
\label{lem:app_pointwise_to_frobenius}
Fix $\e\in(0,1)$ and $\delta\in[0,1]$. Let
$S\in\RR^{k\times n}$ be a random matrix such that, for every fixed
vector $u\in\RR^n$,
\[
    \PP\left[
        \left|\|Su\|^2-\|u\|^2\right|>\e\|u\|^2
    \right]\leq\delta.
\]
Then, for every fixed matrix $V$ with $n$ rows,
\[
    \PP\left[
        \left|\|SV\|_F^2-\|V\|_F^2\right|>\e\|V\|_F^2
    \right]\leq\operatorname{rank}(V)\,\delta.
\]
\begin{proof}
The assertion is immediate for $V=0$. Otherwise, let $u_i$ be its left
singular vectors and $\sigma_i>0$ its nonzero singular values. Then
\[
    \|SV\|_F^2=\sum_i\sigma_i^2\|Su_i\|^2,
    \qquad \|V\|_F^2=\sum_i\sigma_i^2.
\]
There are $\operatorname{rank}(V)$ unit vectors $u_i$. A union bound
ensures that all their squared norms are preserved within $\e$ with
failure probability at most $\operatorname{rank}(V)\delta$. Multiplying
their bounds by $\sigma_i^2$ and summing proves the claim.
\end{proof}
\end{lemma}

\begin{theorem}[Kronecker sketches from fixed-vector guarantees]
\label{thm:app_pointwise_kronecker}
Fix $\e\in(0,1)$ and $\delta_1,\ldots,\delta_d\in[0,1]$.
Let $X$ be fixed and let $S_1,\ldots,S_d$ be independent. Suppose that,
for each $j$ and every fixed $u\in\RR^{n_j}$,
\[
    \PP\left[
        \left|\|S_ju\|^2-\|u\|^2\right|
        >\frac{\e}{4d}\|u\|^2
    \right]\leq\delta_j.
\]
Then
\[
    \PP\left[
        \left|\|\mathcal S(X)\|_F^2-\|X\|_F^2\right|
        >\e\|X\|_F^2
    \right]
    \leq\sum_{j=1}^d\operatorname{rank}(X_{(j)})\,\delta_j.
\]
\begin{proof}
Use the intermediate tensors $X^{(j)}$ from
Theorem~\ref{thm:app_abstract_frobenius_jl}. The first $j-1$ mode
multiplications only form linear combinations of columns in the mode-$j$
unfolding. Hence
\[
    \operatorname{rank}\bigl((X^{(j-1)})_{(j)}\bigr)
    \leq\operatorname{rank}(X_{(j)}).
\]
Conditional on the first $j-1$ factors,
Lemma~\ref{lem:app_pointwise_to_frobenius} therefore bounds the failure
probability of step $j$ by $\operatorname{rank}(X_{(j)})\delta_j$.
The same union bound and multiplication of stepwise norm bounds complete
the proof.
\end{proof}
\end{theorem}

The two composition theorems can be used at any desired distortion
$\e$: evaluate each factor's available failure bound at $\e/(4d)$ and
substitute it for $\delta_j$. For Gaussian and Rademacher factors, the
direct Frobenius bound gives the stronger, rank-independent result above.

\paragraph{The lifted operator.}
Finally, the algebra of Kronecker products transfers unbiasedness and
operator-norm bounds from the factors to the full sketch.

\begin{lemma}[Unbiasedness and norm of Kronecker sketches]
\label{lem:app_kronecker_unbiased_norm}
For independent factors $S_j\in\RR^{k_j\times n_j}$, let
$S=S_d\otimes\cdots\otimes S_1$. If
$\EE[S_j^\top S_j]=I_{n_j}$ for every $j$, then
\[
    \EE[S^\top S]=I_{n_1\cdots n_d}.
\]
For every realization,
\[
    \|S\|=\prod_{j=1}^d\|S_j\|.
\]
In particular, whenever $\|S_j\|^2\leq c_j$ for all $j$,
\[
    \|S\|^2=\|S^\top S\|\leq\prod_{j=1}^dc_j,
    \qquad
    \|S^\top Sx\|^2\leq
    \left(\prod_{j=1}^dc_j\right)^2\|x\|^2
    \quad\text{for all }x.
\]
\begin{proof}
The mixed-product identity gives
\[
    S^\top S=(S_d^\top S_d)\otimes\cdots\otimes(S_1^\top S_1).
\]
Independence lets us take expectations factor by factor, proving
unbiasedness. The singular values of a Kronecker product are products of
the singular values of its factors, which proves the operator-norm
identity. The remaining bounds follow from
$\|S^\top S\|=\|S\|^2$ and
$\|S^\top Sx\|\leq\|S^\top S\|\|x\|$.
\end{proof}
\end{lemma}

These results separate the ingredients needed later: norm and
inner-product preservation control geometry in the compressed space,
while moments and operator norms control the lifted update. A structured
construction can be used in the high-probability convergence theorem
once its parameters provide all three sketch events specified in
Appendix~\ref{appendix:byz}.

\newpage

\newpage


\section{Byzantine-robust Analysis}\label{appendix:byz}

We prove convergence for fixed realized local and trusted objectives satisfying Assumptions~\ref{def:smooth}, \ref{as:data_sim}, and~\ref{as:reference_accuracy}. The constants in these assumptions describe the optimization instance directly; no statistical model for the trusted dataset is imposed.

\subsection{Descent and convergence}

Let $\mathcal F_t$ denote the history before sampling $S_t$, including the fixed objectives, the current iterate, and all previous randomness and messages, and write $\EE_t[\cdot]=\EE[\cdot\mid\mathcal F_t]$. The gradients $g_0^t=\nabla f_0(x^t)$, $g^t=\nabla f(x^t)$, and $g_i^t=\nabla f_i(x^t)$ for $i\in\mathcal G$ are $\mathcal F_t$-measurable. The sketch $S_t$ is independent of $\mathcal F_t$, while Byzantine messages may depend on the current sketch and honest messages. No worker has access to future sketch randomness. All probabilities and expectations below keep the objectives and the initial iterate fixed.

We first derive an estimate for the scalar product that appears in both convergence proofs.

\begin{lemma}\label{lemma2}
Let $x\in\RR^n$ satisfy the reference-accuracy inequality in Assumption~\ref{as:reference_accuracy}. Then, for every $\e\in[0,1)$,
\begin{equation}
    \<\nabla f_0(x),\nabla f(x)>-\e\|\nabla f_0(x)\|\|\nabla f(x)\|
    \geq\frac{1-\eta_1-\e^2}{2}\|\nabla f(x)\|^2-\frac{\eta_0}{2}.
\end{equation}
\begin{proof}
Let $g_0=\nabla f_0(x)$ and $g=\nabla f(x)$. By completing the square,
\begin{eqnarray*}
    \<g_0,g>-\e\|g_0\|\|g\|
    &=&\frac{1}{2}\left(\|g_0\|-\e\|g\|\right)^2
    +\frac{1}{2}\left((1-\e^2)\|g\|^2-\|g_0-g\|^2\right)\\
    &\geq&\frac{1-\eta_1-\e^2}{2}\|g\|^2-\frac{\eta_0}{2},
\end{eqnarray*}
where the last inequality follows from Assumption~\ref{as:reference_accuracy}.
\end{proof}
\end{lemma}

Further, we establish the descent lemma used for convergence both in expectation and with high probability.

\begin{lemma}\label{lemma:scetch}
Let $f_0$ have an $L$-Lipschitz gradient. Then every iteration of Algorithm~\ref{alg:grad_bant} satisfies $f_0(x^{t+1})\leq f_0(x^t)$ and
\begin{eqnarray*}
     f_0(x^{t+1}) &\leq& f_0(x^t)
     -\frac{\gamma\beta}{M}\left\langle g_0^t,\sum_{i\in\mathcal G}u_i^t\right\rangle
     +\frac{L\gamma^2\beta}{2M}\sum_{i\in\mathcal G}\|u_i^t\|^2,
\end{eqnarray*}
where $u_i^t=S_t^\top y_i^t$. For honest workers, $u_i^t=S_t^\top S_tg_i^t$.
\begin{proof}
The initial weights are non-negative and sum to one, and both branches of the weight update preserve these properties. Set
\[
    v^t=\sum_{i=1}^M\omega_i^t\mathbb I_{[\theta_i^t>0]}u_i^t,
    \qquad q_i^t=[\theta_i^t]_0,
    \qquad Q^t=\sum_{i=1}^M q_i^t.
\]
Applying convexity of the squared norm to the vectors
$\mathbb I_{[\theta_i^t>0]}u_i^t$ with weights $\omega_i^t$ gives
\[
    \|v^t\|^2\leq\sum_{i=1}^M\omega_i^t
    \mathbb I_{[\theta_i^t>0]}\|u_i^t\|^2.
\]
Applying smoothness to the actual aggregated update, we obtain
\begin{align*}
    f_0(x^{t+1})
    &\leq f_0(x^t)-\gamma\langle g_0^t,v^t\rangle
        +\frac{L\gamma^2}{2}\|v^t\|^2 
        \leq f_0(x^t)-\gamma\sum_{i=1}^M\omega_i^t
        \mathbb I_{[\theta_i^t>0]}
        \left(\langle g_0^t,u_i^t\rangle
              -\frac{L\gamma}{2}\|u_i^t\|^2\right)\\
    &=f_0(x^t)-\gamma\sum_{i=1}^M\omega_i^tq_i^t
    \leq f_0(x^t).
\end{align*}
If $Q^t>0$, the weight update gives
\begin{align*}
    \sum_{i=1}^M\omega_i^tq_i^t
    &=(1-\beta)\sum_{i=1}^M\omega_i^{t-1}q_i^t
      +\beta\frac{\sum_{i=1}^M(q_i^t)^2}{Q^t} \geq\beta\frac{\sum_{i=1}^M(q_i^t)^2}{Q^t}
    \geq\frac{\beta Q^t}{M}
    \geq\frac{\beta}{M}\sum_{i\in\mathcal G}\theta_i^t.
\end{align*}
The first inequality drops a non-negative term. The second follows from
$(Q^t)^2\leq M\sum_i(q_i^t)^2$, and the last uses
$Q^t\geq\sum_{i\in\mathcal G}q_i^t
\geq\sum_{i\in\mathcal G}\theta_i^t$.
If $Q^t=0$, then $v^t=0$ and
$\sum_{i\in\mathcal G}\theta_i^t\leq0$. Thus, in both cases,
\[
    f_0(x^{t+1})\leq f_0(x^t)
    -\frac{\gamma\beta}{M}\sum_{i\in\mathcal G}\theta_i^t.
\]
Substituting
$\theta_i^t=\langle g_0^t,u_i^t\rangle
-\frac{L\gamma}{2}\|u_i^t\|^2$
proves the claimed bound. The argument uses smoothness of $f_0$ and convexity of the squared norm, without assuming convexity of $f_0$.
\end{proof}
\end{lemma}

We are now ready to present the convergence results. First, we derive convergence in expectation. To cover the different sketches mentioned above, we introduce two assumptions.

\begin{assumption}[Unbiasedness]\label{as:sketch_unbiased}
    The lifted operator is conditionally unbiased if, for every $\mathcal F_t$-measurable vector $v\in\RR^n$,
    \begin{equation*}
        \EE_t[S_t^\top S_tv]=v.
    \end{equation*}
\end{assumption}

\begin{assumption}[Second moment boundedness in expectation]\label{as:sketch_second_moment}
    There exists $\Omega\geq1$ such that, for every $\mathcal F_t$-measurable vector $v\in\RR^n$,
    \begin{equation*}
        \EE_t\|S_t^\top S_tv\|^2\leq\Omega\|v\|^2.
    \end{equation*}
\end{assumption}

\begin{theorem}\label{th:byz_expectation}
Suppose that Assumptions~\ref{def:smooth}, \ref{as:data_sim}, \ref{as:reference_accuracy}, \ref{as:sketch_unbiased}, and~\ref{as:sketch_second_moment} hold. Let $T\geq1$ be an integer, $\beta\in(0,1]$, $\Delta_0=f_0(x^0)-f_0^{\inf}$, and choose $\gamma>0$ such that
\[
    D_{\mathrm E}\eqdef1-\eta_1-L\gamma\Omega(1+\delta_2)>0.
\]
Then, after $T$ iterations of Algorithm~\ref{alg:grad_bant},
\begin{equation}\label{eq:expectation_exact}
    \frac{1}{T}\sum_{t=0}^{T-1}\EE\|\nabla f(x^t)\|^2
    \leq\frac{2M\Delta_0}{\gamma\beta GD_{\mathrm E}T}
    +\frac{\eta_0+L\gamma\Omega\delta_1}{D_{\mathrm E}}.
\end{equation}
In particular, if $0<\gamma\leq\frac{1-\eta_1}{2L\Omega(1+\delta_2)}$, then
\begin{equation}
    \frac{1}{T}\sum_{t=0}^{T-1}\EE\|\nabla f(x^t)\|^2
    \leq\frac{4M\Delta_0}{\gamma\beta G(1-\eta_1)T}
    +\frac{2\eta_0}{1-\eta_1}
    +\frac{2L\gamma\Omega\delta_1}{1-\eta_1}.
\end{equation}
\begin{proof}

Taking conditional expectation in Lemma~\ref{lemma:scetch} and using $\frac{1}{G}\sum_{i\in\mathcal G}g_i^t=g^t$, we obtain
\begin{eqnarray*}
    \EE_t f_0(x^{t+1})
    &\leq& f_0(x^t)-\frac{\gamma\beta G}{M}\<g_0^t,g^t>
       +\frac{L\gamma^2\beta\Omega}{2M}\sum_{i\in\mathcal G}\|g_i^t\|^2\\
    &\leq& f_0(x^t)-\frac{\gamma\beta G}{2M}
       \left((1-\eta_1)\|g^t\|^2-\eta_0\right)
       +\frac{L\gamma^2\beta\Omega G}{2M}
       \left((1+\delta_2)\|g^t\|^2+\delta_1\right)\\
    &=& f_0(x^t)-\frac{\gamma\beta G}{2M}
       \left(D_{\mathrm E}\|g^t\|^2-\eta_0-L\gamma\Omega\delta_1\right).
\end{eqnarray*}
Here, the second inequality uses Lemma~\ref{lemma2} with $\e=0$ and~\eqref{eq:honest_second_moment}: Assumption~\ref{as:reference_accuracy} controls alignment, while Assumption~\ref{as:data_sim} bounds the honest squared norms. Taking full expectation, summing over $t=0,\ldots,T-1$, and using $f_0(x^T)\geq f_0^{\inf}$ gives~\eqref{eq:expectation_exact}. Under the displayed stepsize restriction, $D_{\mathrm E}\geq(1-\eta_1)/2$, which proves the last bound.
\end{proof}
\end{theorem}

For fresh Gaussian and Rademacher matrices, Lemmas~\ref{lem:unbias} and~\ref{lem:variance} give $\Omega=(n+k+1)/k$ and $\Omega=(n+k-1)/k$, respectively. The expectation is over the subsequent algorithmic randomness for fixed objectives satisfying the stated assumptions.

We now derive convergence with high probability. The proof needs three events for each fresh sketch: preservation of one inner product, control of the total compressed honest energy, and an operator-norm bound.

\begin{theorem}[Theorem~\ref{th:byz} from the main part]\label{th:byz_high_probability}
Under the assumptions and notation of Theorem~\ref{th:byz}, except for its stepsize restriction, define
\[
    c_\gamma\eqdef L\gamma\rho(1+\e),
    \qquad D_\gamma\eqdef d-(1+\delta_2)c_\gamma.
\]
More generally, for every $\gamma>0$ with $D_\gamma>0$, we have
\begin{equation}\label{eq:high_probability_exact}
    \frac{1}{T}\sum_{t=0}^{T-1}\|\nabla f(x^t)\|^2
    \leq\frac{2M\Delta_0}{\gamma\beta GD_\gamma T}
    +\frac{\eta_0+c_\gamma\delta_1}{D_\gamma}
\end{equation}
with probability at least $1-\delta$. The stepsize in Theorem~\ref{th:byz} gives~\eqref{eq:main_convergence}.
\begin{proof}
Set $q=\delta/(4T)$ and let $V_t\in\RR^{n\times G}$ be the matrix whose columns are the honest gradients $g_i^t$. Conditional on $\mathcal F_t$, both $V_t$ and the pair $(g_0^t,g^t)$ are fixed. Lemma~\ref{lem:app_one_factor_frobenius} gives
\[
    \PP\left[\|S_tV_t\|_F^2>(1+\e)\|V_t\|_F^2\mid\mathcal F_t\right]
    \leq2e^{-c_1k\e^2}\leq q
\]
when $k\geq C\e^{-2}\log(8T/\delta)$ for a sufficiently large universal constant $C$. The same lemma for a single column gives the fixed-vector JL bound with failure probability at most $q$. Applying Lemma~\ref{lem:app_inner_product} then gives
\[
    \left|\<S_tg_0^t,S_tg^t>-\<g_0^t,g^t>\right|
    \leq\e\|g_0^t\|\|g^t\|
\]
with conditional failure probability at most $2q$. Finally, Lemma~\ref{lem:op_norm_bound} gives $\|S_t\|^2\leq\rho$ with conditional failure probability at most $q$. Thus, with conditional probability at least $1-4q$, all three events hold.

On their intersection, Lemma~\ref{lemma2} implies
\begin{eqnarray*}
    \left\langle g_0^t,\frac{1}{G}\sum_{i\in\mathcal G}u_i^t\right\rangle
    &=&\<S_tg_0^t,S_tg^t> \geq\<g_0^t,g^t>-\e\|g_0^t\|\|g^t\| \geq \frac{d}{2}\|g^t\|^2-\frac{\eta_0}{2}.
\end{eqnarray*}
Moreover, using the operator norm and the Frobenius norm in this order,
\begin{eqnarray*}
    \sum_{i\in\mathcal G}\|u_i^t\|^2
    &=&\|S_t^\top S_tV_t\|_F^2 \leq\|S_t\|^2\|S_tV_t\|_F^2 \leq \rho(1+\e)\|V_t\|_F^2\\
    &\leq&\rho(1+\e)G\left((1+\delta_2)\|g^t\|^2+\delta_1\right).
\end{eqnarray*}
Substituting these two estimates into Lemma~\ref{lemma:scetch} yields
\begin{eqnarray*}
    f_0(x^{t+1})
    &\leq&f_0(x^t)-\frac{\gamma\beta G}{2M}
       \left(d\|g^t\|^2-\eta_0\right)
       +\frac{L\gamma^2\beta\rho(1+\e)G}{2M}
       \left((1+\delta_2)\|g^t\|^2+\delta_1\right)\\
    &=&f_0(x^t)-\frac{\gamma\beta G}{2M}
       \left(D_\gamma\|g^t\|^2-\eta_0-c_\gamma\delta_1\right).
\end{eqnarray*}

Let $\mathcal E_t$ be the intersection of the three sketch events at iteration $t$. Although the iterates are adaptive, their conditional failure probabilities are bounded for every history. Therefore,
\[
    \PP\left[\bigcup_{t=0}^{T-1}\mathcal E_t^c\right]
    \leq\sum_{t=0}^{T-1}\EE\left[\PP(\mathcal E_t^c\mid\mathcal F_t)\right]
    \leq4Tq=\delta.
\]
On $\bigcap_{t=0}^{T-1}\mathcal E_t$, summing the descent inequality and using $f_0(x^T)\geq f_0^{\inf}$ gives
\[
    \frac{\gamma\beta GD_\gamma}{2M}\sum_{t=0}^{T-1}\|g^t\|^2
    \leq\Delta_0+\frac{\gamma\beta GT}{2M}\left(\eta_0+c_\gamma\delta_1\right).
\]
Dividing by $\gamma\beta GD_\gamma T/(2M)$ proves~\eqref{eq:high_probability_exact}. Finally, the stepsize in Theorem~\ref{th:byz} gives $D_\gamma\geq d/2$, which proves~\eqref{eq:main_convergence}.
\end{proof}
\end{theorem}

The probability estimates involve only honest gradients and the trusted reference, all fixed before the fresh sketch is sampled. No geometric preservation event is imposed on Byzantine messages. The same proof applies to another sketch construction if it provides the three conditional events above with the stated failure probabilities and a suitable value of $\rho$; the required dimension must then be established for that construction.

\end{appendixpart}
\end{document}